\documentclass[10pt,twocolumn]{article}

\usepackage[a4paper,margin=2cm]{geometry}
\usepackage{times}

\usepackage[colorlinks,bookmarksopen,bookmarksnumbered,citecolor=red,urlcolor=red]{hyperref}
\usepackage{graphicx}
\usepackage{amsmath,amssymb,bm}
\usepackage{amsthm}
\usepackage{multirow}
\usepackage{subcaption}
\usepackage{url}
\usepackage{hyperref}
\usepackage{ishigaki}
\usepackage{tcolorbox}
\usepackage{booktabs}

\newtheorem{lemma}{Lemma}
\theoremstyle{definition}

\theoremstyle{remark}

\title{Structured Jacobian Construction for Motion Optimization with High-Order Time Derivatives in Multi-Link Systems}

\author{Taiki Ishigaki, Ko Ayusawa and Eiichi Yoshida}

\author{
Taiki Ishigaki \\
Tokyo University of Science \\
\texttt{taiki.ishigaki@rs.tus.ac.jp}
\and
Ko Ayusawa \\
National Institute of Advanced Industrial Science and Technology
\and
Eiichi Yoshida \\
Tokyo University of Science
}
\date{}

\begin{document}

\twocolumn[
\maketitle
\begin{abstract}
This paper presents a novel framework for Jacobian computation in motion optimization problems involving multi-link systems, where physical quantities are represented using higher-order time derivatives.
In motion optimization of robots and humans, cost functions may incorporate higher-order time derivatives, such as jerk or the time variation of forces, to capture smoothness and perceptual characteristics, particularly in the analysis of motion skills and expressive behaviors, thereby necessitating Jacobian computations involving these quantities.
However, such Jacobians are typically computed using numerical or automatic differentiation without explicitly exploiting the underlying multi-link structure, which can lead to increased computational cost and potential numerical instability.
To address this limitation, we propose a structured Jacobian formulation for motion optimization, based on the comprehensive motion computation framework, in which physical quantities and their higher-order time derivatives are systematically represented along the multi-link structure.
The proposed method systematically derives analytical expressions for Jacobians of kinematic and dynamic quantities, including momentum, forces, and joint torques, with respect to generalized coordinates and their higher-order derivatives. The resulting framework is applicable to both direct and inverse optimization.
Through numerical experiments, we demonstrate that the proposed method improves computational efficiency compared to numerical and automatic differentiation, while achieving comparable accuracy. Furthermore, we demonstrate its effectiveness in inverse optimization by recovering cost function weights from motion data.
Together, these results indicate that the proposed formulation provides a scalable and structured computational foundation for motion optimization involving higher-order time derivatives in multi-link systems.
\end{abstract}
\vspace{1em}
]

\section{Introduction}



\begin{figure*}[htp]
    \centering    
    \includegraphics[width=\linewidth]{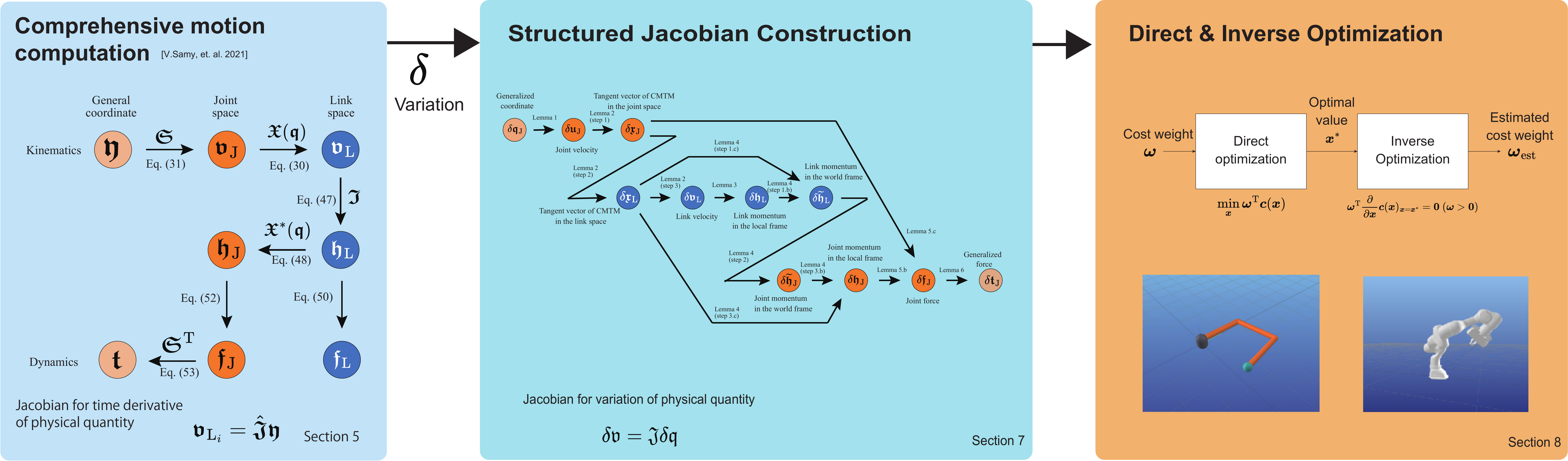}
    \caption{Overview of the structured variation of motion computation.}
    \label{fig:overview}
\end{figure*}

Robot motion computation, which deals with the motion of systems with physical bodies, has significantly contributed to the field of biomechanics \cite{nakamura2005msmodel}. 
A representative example is the use of motion capture to measure human motion and estimate joint and muscle loads using computational human models \cite{delp2007opensim,murai2010musculoskeletal}. 
In parallel, considerable research has been conducted in robotics to understand how humans control their own bodies \cite{alexander1984gaits,flash1985coordination}. 
These studies are typically based on the hypothesis that human motion is governed by an underlying movement principle (i.e., a cost function), and attempt to infer this principle from observed motion data.

To estimate such movement principles, inverse optimal control methods from systems engineering \cite{kalman1964linear} have been widely applied \cite{mombaur2010human,jin2021ioc,bevcanovic2022force,shimizu2024ioc}. 
Unlike learning-based approaches such as reinforcement learning, these methods explicitly handle interpretable cost functions, making them suitable for understanding and predicting human motion.

The inverse optimal control problem assumes that a system operates according to optimal control, and seeks to infer the cost function or control inputs that reproduce observed motion trajectories. 
Several efficient solution methods have been proposed, including the inverse-KKT approach \cite{engler2017ikkt}. 
Although these problems are traditionally studied within the framework of optimal control, closely related formulations also arise in broader optimization contexts. 
In this paper, we adopt the terminology of direct and inverse optimization to refer to these problems.

In conventional formulations for dynamical systems, the cost function typically consists of physical quantities such as position, velocity, acceleration, and force. 
However, when modeling human motion, it is often necessary to consider additional criteria such as smoothness and perceptual qualities. 
As a result, recent studies have incorporated higher-order time derivatives, such as jerk and snap, into the cost function to better capture these characteristics \cite{shimizu2024ioc,berret2011evidence,lin2016human,medrano2019ioc}.

Despite their importance, the systematic computation of Jacobians involving higher-order time-derivative quantities in a structure-aware manner remains limited in conventional robot motion computation. 
In practice, numerical differentiation, automatic differentiation, or analytical derivatives of dynamics algorithms are commonly used. 
While these approaches are effective in many cases, their application to optimization problems involving higher-order time derivatives, particularly in inverse optimization, has not been fully explored. 
Consequently, many existing studies focus on low-dimensional problems such as center-of-mass motion \cite{lin2016human,medrano2019ioc}, and their extension to complex multi-link systems involving higher-order time-derivative quantities remains limited.

On the other hand, recent advances in motion optimization theory have introduced efficient computational frameworks for differentiating the equations of motion with respect to generalized coordinates, velocities, and accelerations \cite{ayusawa2018comprehensive,carpentier2018analytical,singh2022efficient}.
In addition, methods for deriving higher-order time derivatives of motion trajectories and forces have been proposed \cite{samy2021generalized,kumar2021nth,muller2021closed}.
While these methods successfully enable the computation of higher-order time-derivative quantities, they primarily focus on the forward generation of such derivatives.
However, in optimization, what is required is the computation of Jacobians with respect to these higher-order time-derivative variables, which is not explicitly addressed in existing approaches.
As a result, Jacobians involving higher-order quantities are often computed using numerical or automatic differentiation, without fully leveraging the underlying multi-link structure.
In particular, existing approaches do not provide a structure Jacobian computational framework that systematically integrates higher-order time-derivative modeling, multi-link structure, and Jacobian computation for optimization.

In this paper, we propose a structured Jacobian formulation for motion optimization in multi-link systems, based on a comprehensive motion computation framework. 
In the proposed formulation, physical quantities and their higher-order time derivatives are systematically represented and propagated along the multi-link structure via a composition of mappings, enabling direct and structured Jacobian computation with respect to higher-order time-derivative variables in both direct and inverse optimization. 
Figure~\ref{fig:overview} provides an overview of the proposed framework, illustrating how Jacobians arise from structured variations of motion.

\section{Related Work}

The computation of derivatives of kinematic and dynamic quantities has been extensively studied in robotics, particularly in the context of optimization-based motion generation. 
In addition to the evaluation of forward and inverse dynamics, many applications require derivatives of these quantities with respect to joint states, including Jacobians, their time derivatives, and Jacobians of dynamic equations, which are defined over the multi-link structure of articulated systems.

A major line of research focuses on the analytical differentiation of rigid-body dynamics algorithms.
In this approach, classical recursive algorithms such as the recursive Newton–Euler algorithm (RNEA) and the articulated-body algorithm (ABA) are explicitly differentiated to obtain Jacobians of dynamic quantities with respect to joint positions, velocities, and accelerations.
Several works have developed systematic formulations for differentiating these algorithms by exploiting their recursive structure, sparsity, and spatial algebra representations \cite{carpentier2018analytical,paz2024analytical,singh2022efficient}.
These approaches enable efficient derivative computation with linear complexity in the number of degrees of freedom and have been implemented in software libraries such as Pinocchio \cite{carpentier2019pinocchio}.

Beyond direct differentiation of recursive algorithms, related analytical approaches have been developed from various perspectives, including spatial operator algebra \cite{jain1993linearization}, geometric formulations on Lie groups \cite{bos2022efficient}, and sensitivity analysis methods \cite{zhakatayev2023recursive}.
These approaches provide structured representations of linearized dynamics and enable efficient recursive or adjoint-based computation of derivatives of dynamic quantities.

More recently, such analytical formulations have been extended to higher-order derivatives, enabling the use of second-order optimization methods such as DDP and SQP \cite{singh2024second}.

In contrast to approaches based on direct differentiation, another line of work introduces unified representations that embed differential kinematics and dynamics within an extended formulation.
In particular, \cite{ayusawa2018comprehensive} proposed a comprehensive motion transformation matrix (CMTM), which extends conventional spatial transformations to simultaneously represent positions, velocities, and accelerations within a single matrix framework.
This formulation enables differential kinematics and dynamics to be expressed as compositions of mappings, allowing analytical computation of derivatives of various physical quantities in a structured manner.
Building upon this formulation, \cite{samy2021generalized} extended the representation to higher-order time derivatives and derived analytical Jacobians of higher-order time-derivative quantities.
However, these Jacobians are primarily designed to evaluate the time derivatives of physical quantities themselves, and are not directly formulated for systematic Jacobian computation in optimization pipelines.

Higher-order time derivatives of motion and dynamics have also been studied in various contexts, particularly for applications requiring smooth trajectories or higher-order dynamic consistency.
From an optimization perspective, such derivatives appear in cost functions involving jerk, snap, or higher-order quantities.
While existing approaches enable the computation of these quantities or their local derivatives, they are typically developed for specific formulations and do not explicitly address their structured integration within Jacobian computation for optimization.

In parallel, automatic differentiation and differentiable simulation frameworks have been widely adopted to compute derivatives of kinematic and dynamic quantities in a generic manner \cite{newbury2024review}. 
These approaches provide flexibility and ease of implementation, allowing arbitrary models to be differentiated without deriving analytical expressions. 
However, they do not explicitly exploit the underlying multi-link structure, which can lead to increased computational cost and reduced numerical robustness, particularly for higher-order derivatives.

Overall, existing research has addressed individual aspects of derivative computation, including analytical differentiation of dynamics algorithms, unified representations of differential quantities, and higher-order time-derivative modeling.
However, a systematic computational framework that integrates these aspects to enable structured and efficient Jacobian computation for optimization involving higher-order time-derivative physical quantities over multi-link systems remains limited.
\section{Direct / Inverse Optimization\label{sec:dioc}}
The direct optimization problem is formulated as finding the optimization variable $\bm{x}$ that minimizes the following objective function
\begin{align}
    \label{opt_problem}
 \min_{\bm{x}} f(\bm{x}).
\end{align}
In this paper, the objective function $f$  is expressed as a linear combination of predefined sub-cost functions $c_i$ and their associated weights $w_i$, as follows
\begin{align}
    \label{eq:opt_func}
   f(\bm{x}) = f(\bm{x},\bm{w}) 
   =
   \bm{w}^{\trans} \bm{c}(\bm{x})
   =
   \sum_i w_i c_i(\bm{x}).
\end{align}
Here, the vectors $\bm{w}$ and $\bm{c}(\bm{x})$ denote the stacked representations of the weights and the sub-cost functions, respectively. The necessary condition for the optimal solution of \eqref{opt_problem} is given by the Karush–Kuhn–Tucker (KKT) conditions as follows
\begin{align}
\label{kkt_cond}
\frac{\partial f(\bm{x},\bm{w}) }{\partial \bm{x}}
=
\bm{0}.
\end{align}

On the other hand, assuming that the observed variable $\bm{x}$ is a solution to the direct optimization problem, the task of estimating the underlying objective function is referred to as the inverse optimization problem. 
In most cases, the sub-cost functions $\bm{c}$ are assumed to be known, and the problem is formulated as estimating the weight vector $\bm{w}$. In particular, when the observed variable $\bm{x}$ is assumed to satisfy the KKT conditions in \eqref{kkt_cond}, the method for estimating the weights $\bm{w}$ is called the inverse-KKT method \cite{engler2017ikkt}.

Here, the left-hand side of \eqref{kkt_cond} can be rewritten as follows
\begin{align}
    \frac{\partial f(\bm{x},\bm{w})}{\partial \bm{x}}
    =
    \sum_i w_i 
    \frac{\partial c_i(\bm{x})}{\partial \bm{x}}
    =
    \bm{w}^{\trans} \frac{\partial \bm{c}}{\partial \bm{x}}
    =
    \bm{w}^{\trans} \bm{C}
\end{align}
where $\bm{C}$ denotes the matrix whose columns consist of the Jacobian vectors of the sub-cost functions $c_i$.
In the inverse-KKT method, the weight vector $\bm{w}$ is estimated so that it satisfies the KKT condition in \eqref{kkt_cond}. 
Substituting the expression of the Jacobian, this condition can be written as
\begin{align}
    \bm{w}^{\trans} \bm{C} = \bm{0}.
\end{align}
The weight vector $\bm{w}$ is then obtained by solving this homogeneous linear equation under normalization and nonnegativity constraints as
\begin{align}
\label{ikkt}
\text{find } \bm{w} \quad \text{s.t.} \quad 
\bm{w}^{\trans} \bm{C} = \bm{0}, \quad 
\bm{w} \geq \bm{0}, \quad 
\bm{w}^{\trans} \bm{1} = 1.
\end{align}

When a human or a robot is modeled as a multi-link system, the optimization problem for motion typically treats the variable $\bm{x}$ as a time series of joint angles (generalized coordinates). The cost functions are then designed as $\bm{c}(\bm{y}(\bm{x}))$, where $\bm{y}(\bm{x})$ denotes the dynamical variables such as link positions and orientations, spatial velocities, spatial accelerations, and joint torques. Consequently, the computation of the Jacobian matrix $\bm{C}$ requires evaluating the Jacobians of the dynamic variables with respect to the generalized coordinate trajectory, i.e., $\partial \bm{y} / \partial \bm{x}$. The method for computing such Jacobians of general dynamic variables appearing in the equations of motion of multi-link systems is described in \cite{ayusawa2018comprehensive}.




However, in cost functions commonly used for human motion analysis \cite{lin2016human}, higher-order time-derivative variables—such as jerk of the link motion or the time variation of joint torques are included in addition to the dynamical variables described above, in order to evaluate smoothness and perceptual quality of motion.

The computation of higher-order time-derivative variables in multi-link systems has been derived in \cite{samy2021generalized, kumar2021nth}.
These methods obtain variables of arbitrary order by repeatedly applying time differentiation.
However, the Jacobians required in optimization are mathematical Jacobians with respect to higher-order time-derivative variables, and cannot be obtained directly through such repeated differentiation.
In particular, the formulations in \cite{samy2021generalized, kumar2021nth} focus on generating higher-order derivatives themselves, rather than providing the corresponding Jacobian mappings needed for optimization.

In the following sections, we describe the theory of higher-order time-derivative computations for multi-link systems, and in Section \ref{sec:stru_jacob}, we provide derivations of the mathematical Jacobians of representative variables $\bm{y}$ with respect to generalized coordinates and their higher-order time derivatives.
\section{Mathematical Preliminary}

\subsection{Matrix representation for vector operator}

For an arbitrary three-dimensional vector $\bm{u} \in \mathbb{R}^3$,
$\adjrot{\bm{u}}$ denotes the corresponding skew-symmetric matrix representing the cross product.
For an arbitrary six-dimensional vector 
$\bm{u} = [\bm{v}_0^{\text{T}} \ \bm{v}_1^{\text{T}}]^{\text{T}} \in \mathbb{R}^{6} \ (\bm{v}_0, \bm{v}_1 \in \mathbb{R}^3)$,  
we define six-dimensional matrix operator
\begin{align}
    \adjsp{\bm{u}} := 
    \begin{bmatrix}
          \adjrot{\bm{v}_0} & \bm{O}_3 \\
          \adjrot{\bm{v}_1} & \adjrot{\bm{v}_0}
    \end{bmatrix}
    \in
    \mathbb{R}^{6 \times 6}
    .
\end{align}
To generalize the above operation, 
we construct a vector obtained by stacking $k+1$ arbitrary $m$-dimensional vectors as follows.
For 
$\bm{u} = [\bm{u}_0^{\text{T}} \cdots \bm{u}_k^{\text{T}}]^{\text{T}} \in \mathbb{R}^{m(k+1)}$ 
with $(\bm{u}_0, \ldots, \bm{u}_k \in \mathbb{R}^m)$,  
we define
 \begin{align}
\adj{\bm{u}}{m\cdot k} 
&:=
\begin{bmatrix}
    \adj{\bm{u}_0}{m} & \bm{O} & \cdots & \bm{O} \\
    \adj{\bm{u}_1}{m} & \adj{\bm{u}_0}{m} & \ddots & \vdots \\
    \vdots & \ddots & \ddots & \vdots \\
    \adj{\bm{u}_k}{m} & \adj{\bm{u}_{k-1}}{m} & \cdots & \adj{\bm{u}_0}{m} 
\end{bmatrix}.
\end{align}
In the special case where $m = 3$ and $k = 1$, we define  
$\adjsp{\bm{u}} := \adj{\bm{u}}{3\cdot1}$.

For all operators $[\ast \times_n]$ in this paper, we define a corresponding operator $[\ast \hat{\times}_n]$ such that
\begin{align}
[\bm{x} \hat{\times}_n] \bm{y} = [\bm{y} \times_n] \bm{x}.
\end{align}



\subsection{Derivative operator for vectors and matrices\label{app:dev}}

For any vector $\bm{a} \in \mathbb{R}^m$,  
we define a mapping $i$-th derivative of $\bm{a}$ in a given direction $u$ as follows
\begin{align}
   \parvec{\bm{a}}{i}{u}
   := \frac{\partial^{(i)}}{\partial u} \bm{a}.
\end{align}
$u$ can be any direction, such as time or space.
We denote an expression obtained by dividing $\parvec{\bm{a}}{i}{u}$ by $i!$ as follows
\begin{align}
    \parvec{\bm{a}}{i}{u!}
    :=
    \frac{\parvec{\bm{a}}{i}{u}}{i!}.
\end{align}
Furthermore, using $\bm{a}^{(i|u!)}$, we define a mapping $\cmtanvfunc{\bm{a}}{(k|u)}$, which maps $\bm{a}$ to a vector collecting derivatives up to order $k$, as follows
\begin{align}
    \label{cmt_vec_map}
    &\cmtanvfunc{\bm{a}}{(k|u)} \\
    \quad&:=
    \cmtanvec{\bm{a}}{u!}. \notag
\end{align}
Similarly, for a matrix $\bm{A} \in \mathbb{R}^{m \times m}$,  
we define a mapping $\cmtanmfunc{\bm{A}}{(k|u)}$ to a lower block triangular matrix as follows
\begin{align}    
    \label{cmt_mat_map}
    &\cmtanmfunc{\bm{A}}{(k|u)}
    \\
    &=
    \cmtanmat{\bm{A}}{u!}. \nonumber
\end{align}
Here, when the direction of differentiation $|u$ is omitted and the notation is written as $\cmtanvfunc{\bm{a}}{(k)}$ or $\cmtanmfunc{\bm{A}}{(k)}$, it denotes differentiation with respect to the time direction $t$.  
Moreover, when the order of differentiation is omitted and the notation is written as $\cmtanvfunc{\bm{a}}{}$ or $\cmtanmfunc{\bm{A}}{}$, it denotes differentiation up to order $k$.
When a vector is defined as
\begin{align}
\mathfrak{a}_{(k)} := \cmtanvfunc{\bm{a}}{(k)},
\end{align}
the subscript $(k)$ is omitted when it is clear from the context, and we simply write $\mathfrak{a}$.

The first-order derivative of $\cmtanvfunc{\bm{a}}{(k|u)}$ with respect to $u$ can be represented using $\cmtanvfunc{\bm{a}}{(k+1|u)}$, whose derivative orders are shifted by one, together with the coefficient matrix $\natm_{(1:k+1)}$, as
\begin{align}
    \label{dt_diff}
    &\frac{\partial}{\partial u} \cmtanvfunc{\bm{a}}{(k|u)}
    = 
    \begin{bmatrix}
        \zerom_{m(k+1)\times m} & \natm_{(1:k+1)}
    \end{bmatrix}
    \cmtanvfunc{\bm{a}}{(k+1|u)},
    \\
    &\natm_{(1:k+1)}
    :=
    \begin{bmatrix}
        \eyem_m & \bm{O}_m &  \cdots & \bm{O}_m\\
        \bm{O}_m & 2\eyem_m &  \ddots & \vdots\\ 
        \vdots & \ddots & \ddots & \vdots \\
        \bm{O}_m & \cdots & \bm{O}_m & (k+1)\eyem_m
    \end{bmatrix}.
\end{align}

\subsection{Variable explanation}
Variables used in this section are summarized in Table~\ref{tab:variable_table}. 
Six-dimensional spatial quantities, represented using spatial transformation matrices, are denoted in boldface. 
The comprehensive representations, which collect derivatives of a physical quantity from order $0$ to $k$ based on the comprehensive motion transformation matrix, are denoted in Fraktur (German) typeface. 
Additional notations are summarized in Table~\ref{tab:general_variable_table}.

The index or its sets used to describe the tree structure are summarized in Table~\ref{tab:set_childparent}. 
Let $\link_i$ denote the index of the $i$-th link.
Let $\joint_i$ denote the index of the joint connecting link $p(i)$ and link $i$.
Here, $\parent{\link_i}$ denotes the parent index of link $i$, while $\parents{\link_i}$ represents the set of its ancestors (toward the root). 
Similarly, $\child{\link_i}$ denotes the set of children indices of link $i$, and $\children{\link_i}$ represents the set of its descendants (toward the leaves). 
The sets $\parentss{\link_i}$ and $\childrens{\link_i}$ include link $i$ itself.
Let $\sharp_i$ denote either the joint index $J_i$ or the link index $L_i$.


\begin{table}[t]
\centering
\begin{tabular}{c|cc}
Quantity & 6D spatial & Comprehensive \\
\hline
transformation matrix & $\spm$ & $\cmtm$ \\
velocity & $\spv$ & $\cmv$ \\
joint angle & $\jang$ & $\cmjang$ \\
joint angular velocity & $\jvel$ & $\cmjvel$ \\
momentum & $\spmom$ & $\cmmom$ \\
force & $\spforce$ & $\cmforce$ \\
torque & $\torque$ & $\cmtorque$ \\
tangent vector & $\sptanv$ & $\cmtanv$ \\
inertia matrix & $\spinertia$ & $\cminertia$ \\
selection matrix & $\slctm$ & $\cmslctm$ \\
Jacobian matrix & $\jacob$ & $\cmgjacob$ \\
\end{tabular}
\caption{Notation for spatial and comprehensive variables}
\label{tab:variable_table}
\end{table}

\begin{table}[t]
\centering
\begin{tabular}{c|c}
Description & Symbol \\
\hline
zero matrix & $\zerom$ \\
$j$-th unit vector & $\bm{e}_j$ \\
identity matrix & $\eyem$ \\
rotation matrix & $\rot$ \\
position vector & $\pos$ \\
\end{tabular}
\caption{Basic notation}
\label{tab:general_variable_table}
\end{table}

\begin{table}[t]
\centering
\resizebox{\linewidth}{!}{
\begin{tabular}{c|c}
Description & Symbol \\
\hline
index of link $i$ & $\link_i$ \\
index of joint $i$ & $\joint_i$ \\
parent index of link $i$ & $\parent{\link_i}$ \\
ancestor indices of link $i$ & $\parents{\link_i}$ \\
ancestors of $i$ including $i$ & $\parentss{\link_i} := \parents{\link_i} \cup \{\link_i\}$ \\
children indices of link $i$ & $\child{\link_i}$ \\
descendant indices of link $i$ & $\children{\link_i}$ \\
descendants of $i$ including $i$ & $\childrens{\link_i} := \children{\link_i} \cup \{\link_i\}$ \\
the joint index $J_i$ or the link index $L_i$ & $\#_i$
\end{tabular}
}
\caption{Notation for tree-structured index sets}
\label{tab:set_childparent}
\end{table}
\section{Comprehensive Motion Computation\label{sec:cmcomp}}

Figure \ref{fig:quantity_propagation} illustrates the overall structure of the comprehensive motion computation described in this section.
The computation is organized across generalized coordinates, joint space, and link space, covering both kinematic and dynamic quantities.
\begin{figure}
    \centering
    \includegraphics[width=0.9\linewidth]{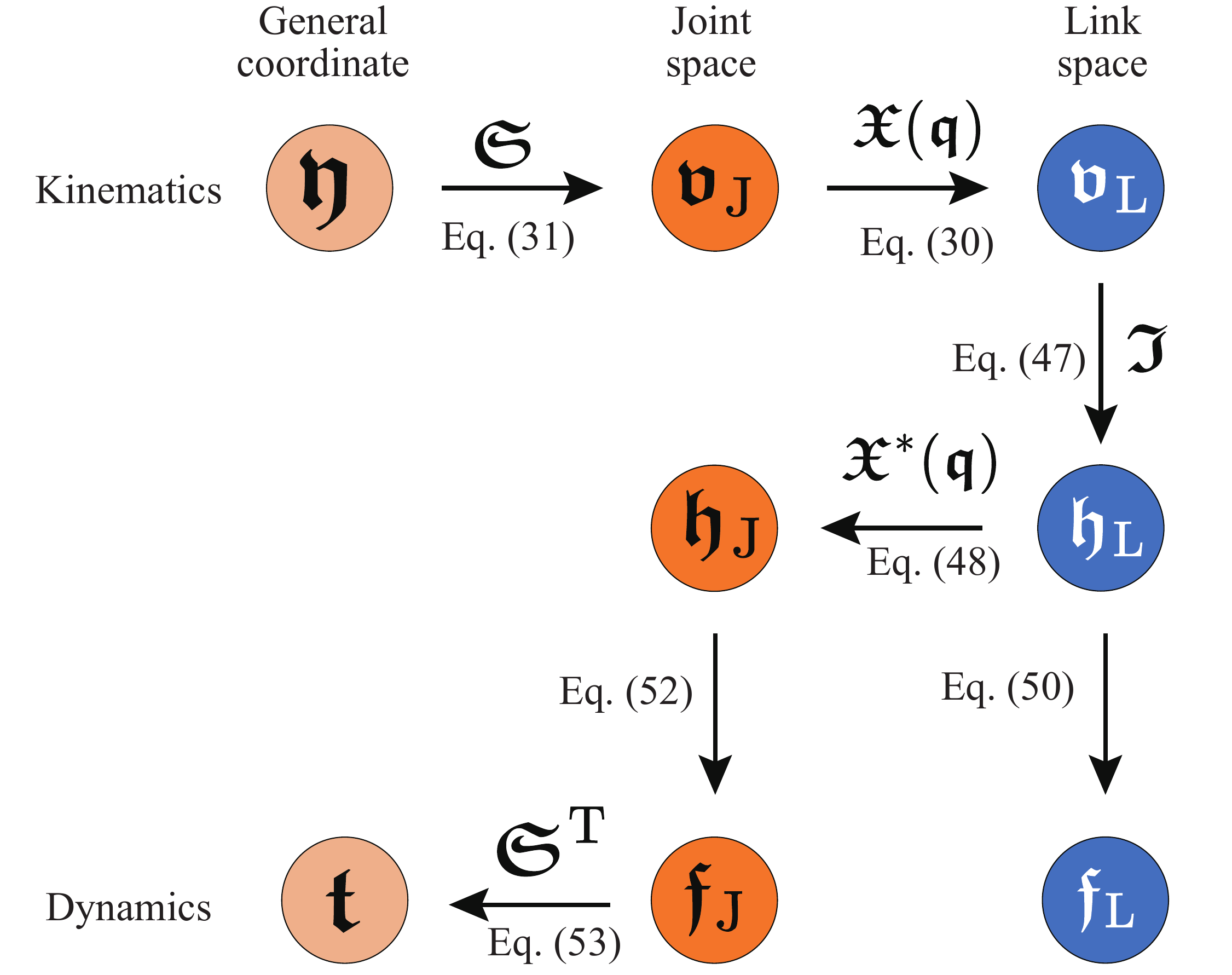}
    \caption{Structure of comprehensive motion computation.
    The figure illustrates the computation flow of kinematic and dynamic quantities in multi-link systems, including generalized coordinates, joint space, and link space, as described in Section~5.}
    \label{fig:quantity_propagation}
\end{figure}

\subsection{Kinematics motion computation}
First, we review the kinematics computation developed in the field of robotics. 
The position and orientation of each link composing a robot can be represented using a spatial transformation matrix defined as follows \cite{featherstone2014rigid}
\begin{align}
    \spm 
    :=
    \spmat
    \in
    SE(3).
\end{align}
where, $\bm{p} \in \mathbb{R}^3$ denotes a position vector, and $\bm{R} \in SO(3)$ denotes a rotation matrix.
Using the spatial transformation matrix, the forward kinematics of the robot can be computed to obtain the position and orientation of each link from the base link to the end-effector, as expressed by the following equation
\begin{align}
    \label{eq:sp_fk}
    \fk{\spm}.
\end{align}
The tangent vector of the spatial transformation matrix with respect to time is defined as follows
\begin{align}
    \label{spv_def}
    \adjsp{\spv} := \spm^{-1} \frac{\partial \spm}{\partial t}.
\end{align}
where $\spv$ is referred to as the spatial velocity \cite{featherstone2014rigid}.
By taking the time derivative of both sides of \eqref{eq:sp_fk}, the following relation is obtained
\begin{align}
    \label{eq:sp_dfk}
    \dfk{\spm}{\spv}
\end{align}
where $\spv_{\joint_i}$ represents the spatial velocity generated by the joint motion. 
It can be computed using the selection matrix $\slctm$, 
and the joint velocity $\jvel$, as expressed below
\begin{align}
    \label{eq:sp_joint_vel}
    \spv_{\joint_i}
    =
    \slctm_{\joint_i} 
    \jvel_{\joint_i}.
\end{align}
where the selection matrix $\slctm_{\joint_i}$ is defined by the joint type, 
\begin{align}
    \label{eq:select}
    \slctm_{\joint_i} :=
    \left\{
    \begin{array}{cll}
        \left [
        \begin{array}{cc}
            \bm{e}_j^{\rm T} & \bm{0}^{\rm T}
        \end{array}
        \right ]^{\rm T}
        &\in \mathbb{R}^{6 \times 1}
        & ({\rm revolute}) \\
        \left [
        \begin{array}{cc}
            \bm{0}^{\rm T} & \bm{e}_j^{\rm T}
        \end{array}
        \right ]^{\rm T}
        &\in \mathbb{R}^{6 \times 1}
        & ({\rm prismatic}) \\
        \left [
        \begin{array}{cc}
            \bm{E} &
            \bm{O}
        \end{array}
        \right ]^{\rm T}
        &\in \mathbb{R}^{6 \times 3}
        & ({\rm spherical}) \\
        \bm{E}
        &\in \mathbb{R}^{6 \times 6}
        & ({\rm floating}). \\
    \end{array}
    \right.
\end{align}
From \eqref{eq:sp_dfk} and \eqref{eq:sp_joint_vel}, the basic Jacobian matrix \cite{khatib1987unified} for link $i$ can be computed as follows
\begin{subequations}
\begin{align}
    \label{sp_basic_jacob}
    \bjacob{\spv}{\jacob}{\jvel}{\spm}{\slctm}.
\end{align}
\end{subequations}
Since $\spv$ cannot be obtained directly as the time derivative of a six-dimensional vector, the matrix $\jacob$ here is referred to as the basic Jacobian rather than a Jacobian matrix.
Note that the basic Jacobian in \cite{khatib1987unified} is defined for world-frame velocity, rather than the local spatial velocity used here.
In the following, most of the matrices that appear are  Jacobian matrices, while some include the basic Jacobian as well. 
For simplicity, however, we refer to all of them collectively as Jacobian matrices.

\subsection{Comprehensive kinematics computation}
When attempting to derive not only \eqref{eq:sp_fk} and \eqref{eq:sp_dfk} but also higher-order differential forward-kinematics equations, the number of terms increases rapidly with the order of differentiation.
To address this issue, the comprehensive motion transformation matrix, which enables the description of motion transformations up to arbitrary orders of time differentiation, has been proposed \cite{ayusawa2018comprehensive,samy2021generalized}.

For the spatial transformation matrix $\spm$, the comprehensive motion transformation matrix (CMTM) $\cmtm$ is defined as follows
\begin{align}
    \cmtm &:=\cmtmat{\cmtpm}
    = \cmtanmfunc{\spm}{(\dindx|t)}, \\
    &
    \label{cmtm_recurrence}
    \begin{cases}
        \cmtpm_0 = \spm \\
        \cmtpm_{\ell+1} 
        = 
        \frac{1}{\ell+1} 
        \sum_{m=0}^\ell \cmtpm_{\ell-m} 
        \adjsp{\spv^{(m|t!)}} 
    \end{cases}
\end{align}
Since the CMTM can be represented by the mapping in \eqref{cmt_mat_map}, its elements can be expressed in terms of the time derivatives of the spatial transformation matrix $\spm$.
However, as shown in \eqref{cmtm_recurrence}, it can also be computed recursively.


Moreover, as with other transformation matrices, CMTM also satisfies the chain rule
\begin{align}
    \label{cmtm_fk}
    \fk{\cmtm}.
\end{align}
As proposed in \cite{samy2021generalized}, forward kinematic computations can be performed not only for velocity but also for higher-order time derivatives.
In addition, the comprehensive spatial velocity $\cmv$, which aggregates the spatial velocity and its derivatives up to order $\dindx$, is defined as follows:
\begin{align}
    \label{cmv_def}
    \cmvec{\cmv}{(\dindx)}
    :=
    \begin{bmatrix}
     \frac{1}{0!}\frac{\partial^{(0)}}{\partial t} \spv^{\text{T}}
     & \cdots
     & \frac{1}{\dindx!}\frac{\partial^{(\dindx)}}{\partial t} \spv^{\text{T}}
    \end{bmatrix}^{\text{T}}
    =
    \cmtanvfunc{\spv}{(\dindx|t)}.
\end{align}
This representation can be naturally extended to other physical quantities.  
In this paper, the stacked representation of a physical quantity and its higher-order time derivatives is referred to as its \textit{comprehensive representation}, and the term \textit{comprehensive} is used in this sense throughout.


In the following, when the notation $\cmv$ is used without an explicit order index, it refers to $\cmvec{\cmv}{(\dindx)}$.
Moreover, as in \eqref{spv_def}, $\cmvec{\cmv}{(\dindx)}$ can be calculated as
\begin{align}
    \label{cmv_def_mat}
    \liealged{\adjcm{\cmv}}{\cmtm}{t}.
\end{align}
Therefore, by multiplying both sides of \eqref{cmv_def_mat} by $\cmtm$,
$\frac{\partial \cmtm}{\partial t}$ can be obtained as follows:
\begin{align}
    \label{derivative_cmtm_by_t}
    \frac{\partial \cmtm}{\partial t}
    =
    \cmtm
    \adjcm{\cmv}.
\end{align}
For the comprehensive spatial velocity $\cmv$, an expression analogous to \eqref{eq:sp_dfk} for spatial velocity $\spv$ holds, as given below
\begin{align}
    \label{cm_dkine}
    \dfk{\cmtm}{\cmv}.
\end{align}
where, $\cmv_{\joint_i}$ can be computed using comprehensive joint angular velocity $\cmvec{\cmjvel}{(\dindx)} := \cmtanvfunc{\jvel}{(\dindx)}$ as follows
\begin{align}
    \label{joint_spatial_cm_trans}
    \cmv_{\joint_i}
    &=
    \cmslctm_{\joint_i}
    \cmjvel_{\joint_i}
    \\
    \cmslctm_{\joint_i} 
    &:= \text{diag}(\begin{bmatrix}\bm{S}_{\joint_i} & \cdots &\bm{S}_{\joint_i}\end{bmatrix}).
\end{align}
As in \eqref{sp_basic_jacob}, the computation of the basic Jacobian matrix can be applied in the same manner
\begin{subequations}
\begin{align}
    \label{cm_basic_jacob}
    \bjacob{\cmv}{{\cmjacobdt}}{\cmjvel}{\cmtm}{\cmslctm}.
\end{align}
\end{subequations}

\subsection{Dynamics computation}
We next describe the dynamic computation.  
For a given link $i$, the spatial momentum $\spmom_{\link_i}$ is given by
\begin{align}
    \label{def_moment}
    \spmom_{\link_i} = \spinertia_{\link_i} \spv_{\link_i}
\end{align}
where $\spinertia_{\link_i}$ is spatial inertia.
We define the transformation matrix for spatial forces and moments as
\begin{align}
    \bm{A}^{\ast}
    :=
    \begin{bmatrix}
        \bm{R} & [\bm{p}\times_3]\bm{R} \\
        \bm{O} & \bm{R}
    \end{bmatrix}
    =
    \bm{A}^{-\text{T}}.
\end{align}

The spatial momentum of link $i$ is further propagated to its parent link through the following relation:
\begin{align}
    \label{moment_dynamics}
    \id{\spm^{\ast}}{\spmom}.
\end{align}

The spatial momentum $\spmom_{\link_i}$ is expressed in the local coordinate frame.  
Other representations include the momentum about the local-frame origin expressed in the world frame, and the momentum about the world-frame origin expressed in the world frame.  
Let $\wmom_{\link_i}$ denote the momentum about the world-frame origin expressed in the world frame.  
The relationship between $\spmom_{\link_i}$ and $\wmom_{\link_i}$ is given by
\begin{align}
    \label{world_local_moment}
    \wmom_{\link_i} = {}^{\text{w}}\bm{A}_{\link_i}^{\ast} \spmom_{\link_i}.
\end{align}
The spatial force in the world frame is given by
\begin{align}
    \label{world_force_moment}
    \wforce_{\link_i} = \frac{\partial}{\partial t}\wmom_{\link_i}.
\end{align}
Since the same transformation holds for forces,
\begin{align}
    \label{world_local_force}
    \wforce_{\link_i} = {}^{\text{w}}\bm{A}_{\link_i}^{\ast} \spforce_{\link_i}.
\end{align}
Left-multiplying \eqref{world_local_force} by ${}^{\text{w}}\bm{A}_{\link_i}^{\ast^{-1}}$, and using \eqref{world_force_moment} and \eqref{world_local_moment}, we obtain
\begin{align}
    \label{link_force_moment}
    \spforce_{\link_i}
    &=
    {}^{\text{w}}\bm{A}_{\link_i}^{\ast^{-1}} \wforce_{\link_i}
    =
    {}^{\text{w}}\bm{A}_{\link_i}^{\ast^{-1}}
    \frac{\partial}{\partial t}
    \Bigl(
        {}^{\text{w}}\bm{A}_{\link_i}^{\ast} \spmom_{\link_i}
    \Bigr)
    \notag\\
    &=
    \frac{\partial \spmom_{\link_i}}{\partial t} 
    +
    \cadjsp{\spv_{\link_i}}
    \spmom_{\link_i}.
\end{align}

For the transmitted force from link $\parent{i}$ to link $i$, the transformation between the world-frame force $\wforce_{\joint_i}$ 
and the local-frame force $\spforce_{\joint_i}$ is given by
\begin{align}
    \wforce_{\joint_i} = {}^{\text{w}}\bm{A}_{\link_i}^{\ast} \spforce_{\joint_i},
\end{align}
$\wforce_{\joint_i}$ satisfies a relation similar to \eqref{world_force_moment}:
\begin{align}
    \wforce_{\joint_i} = \frac{\partial}{\partial t}\wmom_{\joint_i}.
\end{align}

The relationship between $\spforce_{\joint_i}$ and $\spmom_{\joint_i}$ is derived analogously as
\begin{align}
    \label{joint_force_moment}
    \spforce_{\joint_i}
    &=
    {}^{\text{w}}\bm{A}_{\link_i}^{\ast^{-1}} \wforce_{\joint_i}
    =
    {}^{\text{w}}\bm{A}_{\link_i}^{\ast^{-1}}
    \frac{\partial}{\partial t}
    \Bigl(
        {}^{\text{w}}\bm{A}_{\link_i}^{\ast} \spmom_{\joint_i}
    \Bigr)
    \notag\\
    &=
    \frac{\partial \spmom_{\joint_i}}{\partial t} 
    +
    \cadjsp{\spv_{\link_i}}
    \spmom_{\joint_i}.
\end{align}
Note that, also in this case, the second term involves $\cadjsp{\spv_{\link_i}}$, as in \eqref{link_force_moment}, which was derived in \cite{samy2021generalized}. This follows from the time derivative of the transformation matrix ${}^{\text{w}}\bm{A}_{\link_i}^{\ast}$.

Finally, the joint torque $\torque_{\joint_i}$ is obtained as
\begin{align}
    \torque_{\joint_i} = \slctm_{\joint_i}^{\text{T}} \spforce_{\joint_i}.
\end{align}
\subsection{Comprehensive dynamics computation}
The comprehensive link momentum and the corresponding block-diagonal inertia matrix are given as
\begin{align}
    \cmmom_{(\dindx)}
    &:=
    \cmtanvfunc{\spmom}{(\dindx)},
    \\
    \cminertia
    &:=
    \mathrm{diag}([\spinertia \cdots \spinertia]).
\end{align}
In this case, analogous to \eqref{def_moment}, the following expression can be obtained
\begin{align}
    \label{def_cmtm_moment}
    \cmmom_{\link_i} &= \cminertia_{\link_i} \cmv_{\link_i}.
\end{align}
Furthermore, in the same manner as \eqref{moment_dynamics}, the following expression can be obtained
\begin{align}
    \label{cm_moment_dynamics}
    \id{\cmtm^{\ast}}{\cmmom}.
\end{align}
Here, $\cmtm^{\ast}$ is the comprehensive motion transformation matrix corresponding to $\spm^{\ast}$, and is defined as follows
\begin{align}
    \cmtm^{\ast} &:=\cmtmat{\cmtpm^{\ast}}
    = \cmtanmfunc{\spm^{\ast}}{(\dindx|t)}, \\
    &
    \begin{cases}
        \cmtpm_0^{\ast} = \spm^{\ast}, \\
        \cmtpm_{l+1}^{\ast} 
        = 
        \frac{1}{l+1} 
        \sum_{m=0}^l \cmtpm_{l-m}^{\ast} 
        \cadjsp{\spv^{(m|t!)}}.
    \end{cases}
    \notag
\end{align}
The comprehensive link force $\cmvec{{\cmforce_{\link_i}}}{(\dindx)} := \cmtanvfunc{\spforce_{\link_i}}{(\dindx)}$ can be computed in the same manner as \eqref{link_force_moment} as follows
\begin{align}
    \label{cm_force_momentum}
    \cmvec{{\cmforce_{\link_i}}}{(\dindx)}
    &= 
    {}^{\text{w}}{\cmtm_{\link_i}^{\ast}}^{-1}
    \frac{\partial}{\partial t}
    \Bigl(
        {}^{\text{w}}\cmtm_{\link_i}^{\ast} \cmvec{{\cmmom_{\link_i}}}{(\dindx)}
    \Bigr)
    \notag
    \\
    &=
    \frac{\partial}{\partial t}\cmvec{{\cmmom_{\link_i}}}{(\dindx)}
    + 
    \cadjcm{{\cmvec{{\cmv_{\link_i}}}{(\dindx)}}}
    \cmvec{{\cmmom_{\link_i}}}{(\dindx)}
\end{align}
where $\cadjcm{}$ is defined as an operator given by the following expression
\begin{align}
    \cadjcm{\cmvec{{\cmv_{\link_i}}}{(\dindx)}}
    =
    \cadjcm{\bigl(\cmtanvfunc{\spv_{\link_i}}{(\dindx|t)}\bigr)}
    =
    \cmtanmfunc{\cadjsp{\spv_{\link_i}}}{(\dindx|t)}.
\end{align}
The comprehensive force from link $\parent{i}$ to link $i$, $\cmvec{{\cmforce_{{\joint_i}}}}{(\dindx)} := \cmtanvfunc{\spforce_{\joint_i}}{(\dindx)}$ can be computed in the same manner
\begin{align}
    \label{cm_force_momentum_joint}
    \cmvec{{\cmforce_{\joint_i}}}{(\dindx)}
    &= 
    {}^{\text{w}}{\cmtm_{\link_i}^{\ast}}^{-1}
    \frac{\partial}{\partial t}
    \Bigl(
        {}^{\text{w}}\cmtm_{\link_i}^{\ast} \cmvec{{\cmmom_{\joint_i}}}{(\dindx)}
    \Bigr)
    \notag\\
    &=
    \frac{\partial}{\partial t}
    \cmvec{{\cmmom_{\joint_i}}}{(\dindx)}
    + 
    \cadjcm{\cmvec{{\cmv_{\link_i}}}{(\dindx)}}
    \cmvec{{\cmmom_{\joint_i}}}{(\dindx)}.
\end{align}

The comprehensive joint torque $\cmvec{\cmtorque}{(\dindx)} := \cmtanvfunc{\torque}{(\dindx)}$ satisfies
\begin{align}
    \label{cm_torque_force}
    \cmtorque_{\joint_i}
    =
    \cmslctm_{\joint_i}^{\trans}
    \cmforce_{\joint_i}.
\end{align}
As described in Appendix~\ref{app:wbcmtm}, it is also possible to perform whole-body motion computation in a unified manner by stacking the link- and joint-level variables introduced in the comprehensive motion formulation into a single vector.
\section{Elemental Relations for High-Order Jacobian Computation \label{sec:elem_grad}}
In this section, we present a set of elemental relations required for computing higher-order time derivatives in optimization.
These relations describe fundamental transformations and identities that serve as building blocks for the structured Jacobian computation developed in the subsequent sections.

\subsection{Recursive structure and closed-form expansion on multi-link tree graphs \label{sec:grad_rec}}
As in the computation of the basic Jacobian matrix described in Eqs. \eqref{sp_basic_jacob} and \eqref{cm_basic_jacob}, the computation of Jacobian matrices in multi-link systems reduces to solving recursive linear equations defined over a tree-structured graph
\begin{subequations}
\label{eq:structure}
\begin{align}
\bm{x}_i &= \bm{y}_i + \sum_{j \in \lambda(i)} {}^{i}\bm{X}_j \bm{x}_j, \label{eq:structure_main} \\
{}^{i}\bm{X}_j &:= \bm{X}_i^{-1} \bm{X}_j, \label{eq:relative_def} \\
\bm{x}, \bm{y} &\in \mathbb{R}^n, \\
\label{eq:lie_def}
\bm{X} &\in G \subset \mathbb{R}^{n\times n}, \\
\label{eq:set_def}
{i} &\notin \lambda^k(i), \quad \forall i,\ \forall k \ge 1,
\end{align}
\end{subequations}
where \eqref{eq:relative_def} defines the relative transformation from coordinate $i$ to coordinate $j$, $G$ is a matrix Lie group, and $\lambda^k(i)$ denotes the $k$-times application of $\lambda$ to $i$.
The condition \eqref{eq:set_def} implies that the dependency relation defined through $\lambda(i)$ forms a tree-structured (acyclic) graph.
Typical choices of $\lambda(i)$  include parent and child index sets, denoted for example by $\parent{i}$ and $\child{i}$.

Under the above assumptions, the recursive equation \eqref{eq:structure_main} admits the following closed-form expansion due to the acyclic structure of $\lambda(i)$:
\begin{subequations}
\label{recursive_close_form}
\begin{align}
    \bm{x}_i &= \sum_{j \in \hat{\Lambda}(i)} {}^i\bm{X}_j \bm{y}_j,
    \\
    &\Lambda(i) = \bigcup_{k \ge 1} \lambda^k(i),\\
    &\hat{\Lambda}(i) = \{ i \} \cup \Lambda(i).
\end{align}
\end{subequations}
Similarly, in Eqs. \eqref{moment_dynamics} and \eqref{cm_moment_dynamics}, since $\spm^{\ast}$ and $\cmtm^{\ast}$ satisfy the conditions in \eqref{eq:relative_def} and \eqref{eq:lie_def}, and $\children{i}$ satisfies the same structural conditions as $\lambda(i)$ in \eqref{eq:set_def}, the corresponding quantities can be computed in the same manner.

\subsection{Variant of comprehensive motion transformation matrix}
The tangent vector associated with the variation of the comprehensive motion transformation matrix is defined as follows
 \begin{align}
    \label{cmtm_tan_vec}
	  &
    \liealge{\adjcm{\cmvec{\cmtanv}{(\dindx)}}}{\cmtm}
    =
	\cmtmat{\cmptanv},
    \\
    &\begin{cases}
    \cmptanv_0  
    = \adjsp{\delta \bm{a}} := \bm{A}^{-1} \delta \bm{A} \\
    \cmptanv_{{l+1}} 
    = \frac{1}{l+1}
    \left(
	\delta \parvec{\spv}{l}{t!} 
	+ \sum_{m=0}^{l} 
        \adjsp{\cmptanv_{m}}
        \parvec{\spv}{l-m}{t!}
    \right).
    \end{cases}
    \notag
\end{align}
Similarly to \eqref{cm_dkine}, the following equation holds for $\cmtanv$
\begin{align}
    \label{tan_dfk}
    \dfk{\cmtm}{\cmtanv}.
\end{align}\textbf{}
By multiplying $\cmtm$ on both sides of \eqref{cmtm_tan_vec},  we can obtain a variant of $\cmtm$ similar to \eqref{derivative_cmtm_by_t} as follows:
\begin{align}
    \label{cmtm_var}
    \delta \cmtm = \cmtm \adjcm{\cmtanv}.
\end{align}
We obtain a variant of $\cmtm$ acting on an arbitrary vector $\bm{x}$ as follows:
\begin{align}
    \label{cmtm_var_arb_vec}
    \delta \cmtm \bm{x} 
    = \cmtm \adjcm{\cmtanv} \bm{x} 
    = \cmtm [\bm{x} {\widehat{\times}}_{6\cdot\dindx}] \cmtanv.
\end{align}

\subsection{Relationship between different type of derivative}
In motion optimization, the optimization variables consist of the joint velocities and their higher-order time derivatives.  
Accordingly, we establish the correspondence between the tangent vector of the comprehensive motion transformation matrix in \eqref{cmtm_tan_vec} and the variations of the spatial velocity and its higher-order time derivatives.  

Since time differentiation and variation are interchangeable, the following relation holds:
\begin{align}
    \label{tan_time_tan_rel}
    \frac{\partial }{\partial t} \delta \cmtm = \delta \frac{\partial \cmtm}{\partial t}.
\end{align}
Using \eqref{cmtm_var}, the left-hand side of \eqref{tan_time_tan_rel} is given by
\begin{align}
    \label{tan_time_tan_rel_left}
    \frac{\partial }{\partial t} \delta \cmtm 
    &= 
    \frac{\partial }{\partial t} (\cmtm \adjcm{\cmtanv})
    =
    \frac{\partial \cmtm}{\partial t}\adjcm{\cmtanv} + \cmtm \frac{\partial \adjcm{\cmtanv}}{\partial t}
    \notag \\
    &=
    \cmtm \adjcm{\cmv} \adjcm{\cmtanv} + \cmtm \adjcm{\frac{\partial \cmtanv}{\partial t}}.
\end{align}
Similarly, using \eqref{derivative_cmtm_by_t}, the right-hand side is
\begin{align}
    \label{tan_time_tan_rel_right}
    \delta \frac{\partial }{\partial t} \cmtm
    &=
    \delta (\cmtm \adjcm{\cmv})
    =
    \delta \cmtm \adjcm{\cmv} + \cmtm \delta\adjcm{\cmv}
    \notag \\
    &=
    \cmtm \adjcm{\cmtanv}\adjcm{\cmv} + \cmtm \adjcm{\delta\cmv}.
\end{align}
Substituting \eqref{tan_time_tan_rel_left} and \eqref{tan_time_tan_rel_right} into \eqref{tan_time_tan_rel}, left-multiplying by $\cmtm^{-1}$, and using the identity 
$\adjcm{\bm{a}}\adjcm{\bm{b}} - \adjcm{\bm{b}}\adjcm{\bm{a}} = \adjcm{(\adjcm{\bm{a}}\bm{b})}$,
we obtain
\begin{align}
    \label{lie_alg_time_var}
    \delta \cmvec{\cmv}{(\dindx)} = \adjcm{\cmvec{\cmv}{(\dindx)}}\cmvec{\cmtanv}{(\dindx)} + \frac{\partial}{\partial t}
    \cmvec{\cmtanv}{(\dindx)}
    .
\end{align}
Using \eqref{dt_diff}, we obtain
\begin{align}
    \label{dt_cmtanv}
    \frac{\partial \cmvec{\cmtanv}{(k)}}{\partial t} =     
    \begin{bmatrix}
        \zerom_{m(k+1)\times m} & \natm_{(1:k+1)}
    \end{bmatrix}
    \cmvec{\cmtanv}{(k+1)}.
\end{align}
By using \eqref{dt_cmtanv}, 
Eq.~\eqref{lie_alg_time_var} can be regarded as the following linear transformation
\begin{align}
    \label{cm_tan_map}
    &\delta\cmvec{\cmv}{(\dindx)}
    =
    \bm{\mathfrak{U}}(\adjcm{\cmvec{\cmv}{(\dindx)}})
    \cmtanv_{(\dindx+1)},
    \\
    &\bm{\mathfrak{U}}(\adjcm{\cmvec{\cmv}{(\dindx)}})
    := \notag
    \\
    &\quad\left(
    \begin{bmatrix}
        \adjcm{\cmvec{\cmv}{(\dindx)}} & \zerom_{6(\dindx+1)\times 6}
    \end{bmatrix}
    +
    \begin{bmatrix}
        \zerom_{6(\dindx+1)\times 6} & \natm_{(1:\dindx+1)}
    \end{bmatrix}
    \right).
    \notag
\end{align}
We newly define the vector $\cmptanv$ by concatenating $\delta \bm{a}$ and $\delta \cmvec{\cmv}{(\dindx)}$ as
\begin{align}
    \cmvec{\cmptanv}{(\dindx)}
    :=
    \begin{bmatrix}
        \delta \bm{a} \\
        \delta \cmvec{\cmv}{(\dindx)}
    \end{bmatrix}.
\end{align}
By adding the identity $\delta \bm{a} = \delta \bm{a}$ to \eqref{cm_tan_map},  
the following expression is obtained
\begin{align}
    \label{cm_tan_map_2}
    &\quad\quad\quad\quad\quad\quad\quad\quad
    \cmvec{\cmptanv}{(\dindx)}
    =
    \cmvec{\cmtanmap}{(\dindx+1)}
    \cmvec{\cmtanv}{(\dindx+1)},
    \\
    &\cmvec{\cmtanmap}{(\dindx+1)}
    :=
    \begin{bmatrix}
        \begin{bmatrix}
        \eyem_6 & \zerom_{6\times 6(\dindx+1)}
        \end{bmatrix} \\
        \bm{\mathfrak{U}}(\adjcm{\cmvec{\cmv}{(\dindx)}})
    \end{bmatrix}
    \\
    &=
    \begin{bmatrix}
        \zerom_{6\times6(\dindx+1)} & \zerom_{6} \\
        \adjcm{\cmv_{(\dindx)}} & \zerom_{6(\dindx+1)\times 6}
    \end{bmatrix}
    +
    \begin{bmatrix}
        \eyem_6 & \zerom_{6\times 6(\dindx+1)} \\
        \zerom_{6(\dindx+1) \times 6} & \natm_{(1:\dindx+1)}
    \end{bmatrix}.
    \notag
\end{align}
The explicit forms of $\cmvec{\cmtanmap}{(\dindx+1)}$ and its inverse are given in Appendix \ref{apdx_cmtm_tan_map}.

\subsection{Coordinate transformation of force and moment variations}
As discussed in Section~\ref{sec:grad_rec}, the recursive computation of the momentum in \eqref{cm_moment_dynamics} satisfies the conditions in \eqref{eq:structure_main}.  
However, when considering its variation, the variation of $\cmtm$ is involved, and the resulting expression no longer satisfies the conditions in \eqref{eq:structure_main}.

To address this issue, we reformulate the recursion using the comprehensive momentum expressed in the world frame, 
\begin{align}
    \label{eq:w_cm_mom}
    \wcmmom_{\#_i} := {}^{\world}\cmtm_{\link_i}^{\ast} \cmmom_{\#_i} (\# = \link \ \text{or} \ \joint).
\end{align}
In this representation, \eqref{cm_moment_dynamics} can be written as
\begin{align}
    \label{wcm_moment_dynamics}
    \wid{\wcmmom}.
\end{align}
Its variation is then given by
\begin{align}
    \label{wcm_moment_dynamics_var}
    \wid{\delta\wcmmom}.
\end{align}
In this form, $\wcmmom_{\joint_i}$ can be interpreted as being premultiplied by the identity matrix, resulting in a linear propagation structure.
A similar linear relationship is also employed in \cite{jain1993linearization}.
This simplification arises because, in the world frame, the propagation of momentum between links is expressed as a direct summation, eliminating the need for relative coordinate transformations.
Since the identity matrix satisfies the conditions in \eqref{eq:relative_def} and \eqref{eq:lie_def}, the same recursive formulation applies without modification.

From \eqref{cm_force_momentum} and \eqref{cm_force_momentum_joint}, it can be observed that the subscripts of the comprehensive spatial velocity appearing before the moment correspond not to relative velocities, but to velocities of each frame with respect to the world coordinate system.
This is because, in the transformation from moment to force, both the joint space and the link space are first mapped into the world coordinate system \eqref{eq:w_cm_mom}, and the resulting form arises from taking the time derivative in that representation.

The transformation from moment to force can be written in a unified manner for both joint space and link space.
Using \eqref{dt_diff}, similarly to \eqref{cm_tan_map}, \eqref{cm_force_momentum} can be expressed as the following matrix mapping from moment to force
\begin{align}
    \label{link_mom_to_force}
    &\cmvec{{\cmforce_{\link_i}}}{(\dindx)}
    =
    \bm{\mathfrak{U}}(\cadjcm{\cmvec{{\cmv_{\link_i}}}{(\dindx)}{}})
    \cmvec{{\cmmom_{\link_i}}}{(\dindx+1)}.
\end{align}
The comprehensive force from link $\parent{i}$ to link $i$, $\cmvec{{\cmforce_{\joint_i}}}{(\dindx)}$, can be computed in the same manner
\begin{align}
    \label{joint_mom_to_force}
    \cmvec{{\cmforce_{\joint_i}}}{(\dindx)}
    =
    \bm{\mathfrak{U}}(\cadjcm{\cmvec{{\cmv_{\link_i}}}{(\dindx)}{}})
    \cmvec{{\cmmom_{\joint_i}}}{(\dindx+1)}.
\end{align}

\section{Structured Jacobian Construction for High-Order Time Derivative Systems\label{sec:stru_jacob}}
\begin{figure*}[tbp]
    \centering    
    \includegraphics[width=0.9\linewidth]{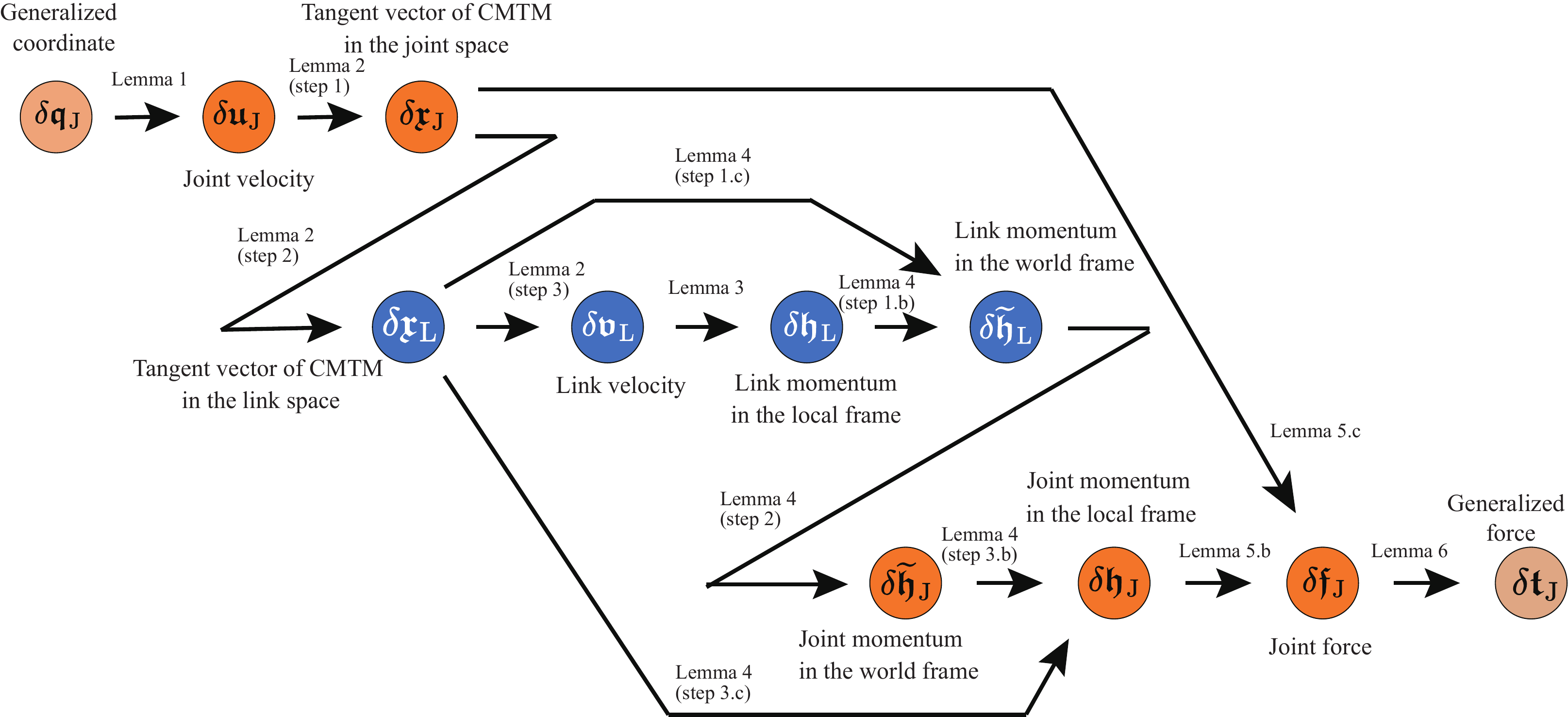}
    \caption{Schematic illustration of the Jacobian computation as a chain of transformations.
            Each node represents a physical quantity, and each edge denotes a linear mapping between variations.The overall Jacobian is obtained by composing these mappings along the chain.}
    \label{fig:gradinet_chain}
\end{figure*}

\begin{table*}
    \centering
    \caption{Summary of Jacobians and their associated physical quantities. Each Jacobian maps variations of generalized coordinates to variations of the corresponding quantity in the structured propagation framework.}
    \begin{tabular}{l c c c ll}
        \toprule
        Quantity & Quantity symbol $\delta \bm{\mathfrak{z}}$ & Jacobian $\cmgjacobarrow{\cmjang}{\bm{\mathfrak{z}}}$ & Eq. & Derived from & \\
        \midrule
        Joint velocity & $\cmptanv_{\joint_i}$ & $\cmptanjacobj{i}$ & \eqref{cm_tan_joint_jacob} & 
        Lemma~\ref{lem:j_vel} & 
        \eqref{joint_spatial_cmtan_trans} \\
        
        Joint tangent & $\cmtanv_{\joint_i}$ & $\cmtanjacobj{i}$ & \eqref{cm_tan_joint_jacob} & 
        Lemma~\ref{lem:cm_vel} (Step 1) & 
        \eqref{var_cmv_1}\\
        
        Link tangent & $\cmtanv_{\link_i}$ & $\cmtanjacob{i}{}$ & \eqref{var_tan_cmtm_link_jacob} & 
        Lemma~\ref{lem:cm_vel} (Step 2) &
        \eqref{var_cmv_2}\\
        
        Link velocity & $\delta \cmv_{\link_i}$ & $\cmjacob{i}$ & \eqref{cm_vel_link_jacob} & 
        Lemma~\ref{lem:cm_vel} (Step 3) &
        \eqref{var_cmv_3}\\
        
        Link momentum & $\delta \cmmom_{\link_i}$ & $\cmlmomjacob{i}$ & \eqref{link_mom_jacob} & 
        Lemma~\ref{lem:lmom} &
        \eqref{delta_cmtm_moment} \\
        Link momentum (world) & $\delta \wcmmom_{\link_i}$ & $\wcmlmomjacob{i}$ & \eqref{link_world_mom_jacob} & 
        Lemma~\ref{lem:mom} (Step 1) &
        \eqref{lem_var_wmom}\\
        
        Joint momentum (world) & $\delta {\wcmmom}_{\joint_i}$ & $\wcmjmomjacob{i}$ & \eqref{joint_world_mom_jacob} & 
        Lemma~\ref{lem:mom} (Step 2) &  
        \eqref{var_link_world_mom_local_mom} \\
        
        Joint momentum & $\delta \cmmom_{\joint_i}$ & $\cmjmomjacob{i}$ & \eqref{joint_mom_jacob} & 
        Lemma~\ref{lem:mom} (Step 3) &
        \eqref{j_cmt_mom_var} \\
        
        Link force & $\delta \cmforce_{\link_i}$ & $\cmlforcejacob{i}$ & \eqref{cm_link_force_jacob} & 
        Lemma~\ref{lem:mom_force} &
        \eqref{var_cmmom_cmforce}\\
        
        Joint force & $\delta \cmforce_{\joint_i}$ & $\cmjforcejacob{i}$ & \eqref{cm_joint_force_jacob} & 
        Lemma~\ref{lem:mom_force} &
        \eqref{var_cmmom_cmforce} \\
        
        Joint torque & $\delta \cmtorque_{\joint_i}$ & $\cmtorquejacob{i}$ & \eqref{torque_jacob} & 
        Lemma~\ref{lem:torque} &
        \eqref{var_cm_torque} \\
        \bottomrule
    \end{tabular}
    \label{tab:jacobian_summary}
\end{table*}

In this section, we construct Jacobians of physical quantities with respect to generalized coordinates and their higher-order derivatives.
Building upon the equations introduced in Section~\ref{sec:cmcomp}, whose variations are considered in this section, and the elemental relations presented in Section~\ref{sec:elem_grad}, the Jacobian computation is formulated as a composition of linear mappings aligned with the computational structure.

Figure~\ref{fig:quantity_propagation} summarizes the forward computation of physical quantities.
Figure~\ref{fig:gradinet_chain} illustrates the corresponding propagation of variations.
The latter is derived from the variation of the former, where each operation is replaced by its associated linear mapping.
This correspondence enables a unified and structured construction of Jacobians.

\subsection{Structured Jacobian propagation}
As illustrated in Fig.~\ref{fig:gradinet_chain}, the propagation of variations across physical quantities forms a chain of transformations.
Each transformation defines a linear mapping, and the overall Jacobian is obtained by composing these mappings along the chain.

Specifically, the variation is propagated through the following stages, as shown in Fig.~\ref{fig:quantity_propagation}.
Each stage defines a linear mapping between variations of physical quantities, represented by its Jacobian.

We first consider the propagation to the comprehensive link spatial velocity.
The variation is given by
\begin{equation}
\delta \cmv_{\link_i} = \cmjacob{i} \delta \cmvec{\cmjang}{(k)},
\end{equation}
where $\cmjacob{i}$ denotes the corresponding Jacobian.
Here, $\cmvec{{\cmjang_{\joint_i}}}{(k)}$ is defined by the concatenation of $ \jang_{\joint_i}$ and $ \cmjvel_{\joint_i}$ as
\begin{align}
    \cmvec{{\cmjang_{\joint_i}}}{(k)}
    :=
    \begin{bmatrix}
        \jang_{\joint_i} \\
        \cmvec{{\cmjvel_{\joint_i}}}{(k)}
    \end{bmatrix}
\end{align}
By collecting all joint components, $\cmjang_\joint$ is defined as follows:
\begin{align}
    \cmjang_\joint
    &=
    \begin{bmatrix}
        \cmjang_{\joint_0}^\trans & 
        \cdots & \cmjang_{\joint_n}^\trans
    \end{bmatrix}^\trans.
\end{align}

The same structure applies to all other physical quantities.
Accordingly, their variations can be expressed in a unified form as
\begin{equation}
\delta \bm{\mathfrak{z}} = \cmgjacobarrow{\cmjang_\joint}{\bm{\mathfrak{z}}} \delta \cmvec{{\cmjang_\joint}}{(k)},
\end{equation}
where $\bm{\mathfrak{z}}$ represents spatial velocity, momentum, force, or torque, and $\cmgjacobarrow{\cmjang_\joint}{\bm{\mathfrak{z}}}$ denotes the Jacobian associated with $\bm{\mathfrak{z}}$.

A summary of the Jacobians and their associated physical quantities is provided in Table~\ref{tab:jacobian_summary}. 
Each row corresponds to a specific quantity and its Jacobian, explicitly showing the mapping from variations of generalized coordinates to variations of the quantity.
To construct these Jacobians in a structured manner, we first introduce a set of fundamental Jacobian mappings as lemmas. 
Each lemma describes a linear mapping between variations of intermediate physical quantities, which serve as building blocks for the overall Jacobian construction.
The lemmas presented in this section are derived by integrating the variation of the comprehensive motion computation introduced in Section~\ref{sec:cmcomp} (Fig.~\ref{fig:quantity_propagation}) with the elemental mappings defined in Section~\ref{sec:elem_grad}. 
By composing these mappings, the Jacobians associated with target physical quantities can be systematically constructed within a unified framework.

\subsection{Elemental Jacobian mappings}
The following lemmas are derived as variational forms of the comprehensive motion computations introduced in Section~\ref{sec:cmcomp}.
Table~\ref{tab:lemma_structure} summarizes the corresponding comprehensive motion computation equations and the related expressions associated with each lemma.

\begin{lemma}
\label{lem:j_vel}
The Jacobian mapping from variations of generalized coordinates to variations of joint spatial velocity is given by
\begin{subequations}
\begin{align}
    \label{joint_spatial_cmtan_trans}
    \cmvec{{\cmptanv_{\joint_i}}}{(\dindx)}
    &=
    \cmgjacobarrow{\cmjang_{\joint_i}}{\cmptanv_{\joint_i}}
    \delta \cmvec{{\cmjang_{\joint_i}}}{(\dindx)}, 
    \\
    &
    \cmgjacobarrow{\cmjang_{\joint_i}}{\cmptanv_{\joint_i}}
    =
    \cmslctm_{\joint_i}.
\end{align}
\end{subequations}
\end{lemma}
\proof
This result follows directly from the variation of \eqref{joint_spatial_cm_trans}.

\begin{lemma}
\label{lem:cm_vel}
The Jacobian mapping from joint spatial velocity to link spatial velocity is given by the following sequence of linear mappings:
\begin{itemize}
    \item \textbf{Step 1: Tangent transformation.} 
    Map the variation to the tangent space at the joint:
    \begin{subequations}
    \begin{align}
        \label{var_cmv_1}
        \cmvec{{\cmtanv_{\joint_i}}}{(\dindx+1)}
        &= 
        \cmgjacobarrow{\cmptanv_{\joint_i}}{\cmtanv_{\joint_i}}
        \cmvec{{\cmptanv_{\joint_i}}}{(\dindx)},
        \\
        &
        \cmgjacobarrow{\cmptanv_{\joint_i}}{\cmtanv_{\joint_i}}
        = 
        \cmvec{\cmtanmap}{(\dindx+1)}^{-1}
    \end{align}
    \end{subequations}
    \item \textbf{Step 2: Kinematic propagation.}
    Propagate the tangent vector along the kinematic chain:
    \begin{subequations}
    \begin{align}
        \label{var_cmv_2}
        \cmtanv_{\link_i} 
        =
        \sum_{j \in \parentss{i}}
        \cmgjacobarrow{\cmtanv_{\joint_j}}{\cmtanv_{\link_i}}
        \cmtanv_{\joint_j} 
        ,
        \\
        \cmgjacobarrow{\cmtanv_{\joint_j}}{\cmtanv_{\link_i}}
        =
        {}^{\link_i}\cmtm_{\link_{\parent{\joint_i}}}.
    \end{align}
    \end{subequations}
    \item \textbf{Step 3: Velocity recovery.}
    Recover the spatial velocity variation at the link:
    \begin{subequations}
    \begin{align}
        \label{var_cmv_3}
        \delta \cmvec{{\cmv_{\link_i}}}{(\dindx)} &= \cmgjacobarrow{\cmtanv_{\link_i}}{\cmv_{\link_i}} \cmvec{{\cmtanv_{\link_i}}}{(\dindx+1)}, \\
        &\cmgjacobarrow{\cmtanv_{\link_i}}{\cmv_{\link_i}} 
        =  
        \bm{\mathfrak{U}}(\adjcm{\cmvec{{\cmv_{\link_i}}}{(\dindx)}}).
    \end{align}
    \end{subequations}
\end{itemize}
\end{lemma}
\proof
Steps 1 and 3 are obtained by reorganizing the previously derived relations in \eqref{cm_tan_map_2} and \eqref{cm_tan_map}.
Step 2 is obtained by applying the closed-form expansion \eqref{recursive_close_form}, since \eqref{tan_dfk} has the same recursive structure as \eqref{eq:structure_main}, allowing the propagation to be expressed as a summation over the ancestral joints.

\begin{lemma}
\label{lem:lmom}
The Jacobian mapping from variations of spatial velocity to variations of momentum in link space is given by
\begin{subequations}
\begin{align}
    \label{delta_cmtm_moment}
    \delta {\cmmom_{\link_i}} 
    = 
    \cminertia_{\link_i} 
    \delta {\cmv_{\link_i}}, \\
    \cmgjacobarrow{\cmv_{\link_i}}{\cmmom_{\link_i}}
    =
    \cminertia_{\link_i}.
\end{align}
\end{subequations}
\end{lemma}
\proof
This follows from the variation of \eqref{def_cmtm_moment}, noting that $\cminertia_{\link_i}$ is constant.

\begin{lemma}
    \label{lem:mom}
    mapping from variation of link momentum to variation of joint momentum is given by the following sequence of linear mappings:
    \begin{itemize}
        \item \textbf{Step 1: World transformation.} Transform the variation of link momentum to the world frame:
        \begin{subequations}
        \begin{align}
            \label{lem_var_wmom}
            &\delta \wcmmom_{\link_i}
            =
             \cmgjacobarrow{\cmmom_{\link_i}}{\wcmmom_{\link_i}}
            \delta \cmmom_{\link_i}
            +
            \cmgjacobarrow{\cmtanv_{\link_i}}{\wcmmom_{\link_i}}
            \cmtanv_{\link_i},
            \\
            &\quad\cmgjacobarrow{\cmmom_{\link_i}}{\wcmmom_{\link_i}}
            =
            {}^{\world}\cmtm_{\link_i}^{\ast}, 
            \\
            &\quad\cmgjacobarrow{\cmtanv_{\link_i}}{\wcmmom_{\link_i}}
            =
            {{}^{\world}\cmtm_{\link_i}^{{\ast}}}
            [\cmmom_{\link_i} \widehat{\times}^{\ast}_{6k}].
        \end{align}    
        \end{subequations}
       \item \textbf{Step 2: Momentum propagation.}
        Propagate the transformed world momentum variations from link space to joint space over the child links:
        \begin{subequations}
        \begin{align}
            \label{var_link_world_mom_local_mom}
            \delta \wcmmom_{\joint_i} 
            &=
            \sum_{j \in \childrens{i}}
            \cmgjacobarrow{\wcmmom_{\link_j}}{\wcmmom_{\joint_i}}
            \delta \wcmmom_{\link_j}, \\
            \cmgjacobarrow{\wcmmom_{\link_j}}{\wcmmom_{\joint_i}}
            &=
            \bm{E}.
        \end{align}
        \end{subequations}

        \item \textbf{Step 3: Local recovery.}
        Recover the local frame representation:
        \begin{subequations}
        \begin{align}
            \label{j_cmt_mom_var}
            \delta \cmmom_{\joint_i}
            &=
            \cmgjacobarrow{\wcmmom_{\joint_i}}{\cmmom_{\joint_i}}
            \delta \wcmmom_{\joint_i}
            +
            \cmgjacobarrow
            {\cmtanv_{\link_i}}
            {\cmmom_{\joint_i}}
            \cmtanv_{\link_i},
            \\
            \quad&
            \cmgjacobarrow{\wcmmom_{\joint_i}}{\cmmom_{\joint_i}}
            =
            {{}^{\world}\cmtm_{\joint_i}^{{\ast}^{-1}}}
            ,\\
            \quad&
            \cmgjacobarrow
            {\cmtanv_{\link_i}}
            {\cmmom_{\joint_i}}
            =
            -
            \cadjcm{({{}^{\world}\cmtm_{\joint_i}^{{\ast}^{-1}}}
            \wcmmom_{\joint_i})}.
        \end{align}
        \end{subequations}
    \end{itemize}
\end{lemma}
\proof
This mapping is constructed by taking the variation of the momentum transformations across coordinate frames and their accumulation in the world frame.


\paragraph{Step 1 (World transformation).}
We begin with the variation of the transformation from the link frame to the world frame in \eqref{eq:w_cm_mom}:
\begin{equation}
\label{cm_world_local_mom_variant}
\delta \wcmmom_{\link_i}
=
\delta \left({}^{\world}\cmtm_{\link_i}^{\ast}\right) \cmmom_{\link_i}
+
{}^{\world}\cmtm_{\link_i}^{\ast} \delta \cmmom_{\link_i}.
\end{equation}
Applying \eqref{cmtm_var_arb_vec}, the first term becomes
\begin{equation}
\label{var_cm_ast_world}
\delta \left({}^{\world}\cmtm_{\link_i}^{\ast}\right) \cmmom_{\link_i}
=
{}^{\world}\cmtm_{\link_i}^{\ast}
[\cmmom_{\link_i} \widehat{\times}^{\ast}_{6\cdot\dindx}]
\cmtanv_{\link_i}.
\end{equation}
Substituting \eqref{var_cm_ast_world} into \eqref{cm_world_local_mom_variant}, we obtain the first mapping in Lemma~\ref{lem:mom}.

\paragraph{Step 2 (Momentum propagation).}
Next, the second mapping in Lemma~\ref{lem:mom} is derived by considering the accumulation of the variation of the moment in the world frame, as expressed in \eqref{wcm_moment_dynamics_var}, and applying \eqref{recursive_close_form}.

\paragraph{Step 3 (Local recovery).}
Finally, we recover the local joint momentum from its world representation.
This is derived from the inverse transformation
$\cmmom_{\joint_i} = ({}^{\world}\cmtm_{\link_i}^{\ast})^{-1} \wcmmom_{\joint_i}$.
Taking its variation yields
\begin{align}
    \label{momentum_joint_relation_inv_var}
    \delta \cmmom_{\joint_i}
    =
    {{}^{\world}\cmtm_{\link_i}^{{\ast}^{-1}}}
    \delta \wcmmom_{\joint_i}
    +
    \delta {{}^{\world}\cmtm_{\link_i}^{{\ast}^{-1}}}
    \wcmmom_{\joint_i}.
\end{align}
The second term represents the variation of the inverse transformation.
Applying \eqref{cmtm_var_arb_vec} and \eqref{var_inv_trans_vec} in Appendix~\ref{app:inv_var}, we obtain
\begin{align}
    \label{inv_trans_momentum_var}
    \delta {{}^{\world}\cmtm_{\link_i}^{{\ast}^{-1}}}
    \wcmmom_{\joint_i}
    =
    -
    \cadjcmt{({{}^{\world}\cmtm_{\link_i}^{\ast}}^{-1}\wcmmom_{\joint_i})}
    \cmtanv_{\link_i}.
\end{align}
Substituting this into \eqref{momentum_joint_relation_inv_var} directly yields the third mapping in Lemma~\ref{lem:mom}.

\begin{lemma}
\label{lem:mom_force}
The Jacobian mapping from $\delta \cmmom_{\#_i}$ and $\delta \cmv_{\link_i}$ to $\delta \cmforce_{\#_i}$ is given by
\begin{subequations}
\begin{align}
    \label{var_cmmom_cmforce}
    &\delta \cmvec{{\cmforce_{\#_i}}}{(\dindx)}
    =
    \cmgjacobarrow{\cmmom_{\#_i}}{\cmforce_{\#_i}}
    \delta \cmvec{{\cmmom_{\#_i}}}{(\dindx+1)}
    +
    \cmgjacobarrow{\cmv_{\link_i}}{\cmforce_{\#_i}}
    \delta \cmvec{{\cmv_{\link_i}}}{(\dindx+1)},
    \\
    &
    \cmgjacobarrow{\cmmom_{\#_i}}{\cmforce_{\#_i}}
    =
    \bm{\mathfrak{U}}(\cadjcm{\cmvec{{\cmv_{\link_i}}}{(\dindx)}{}}),
    \\
    &
    \cmgjacobarrow{\cmv_{\link_i}}{\cmforce_{\#_i}}
    =
    \begin{bmatrix}
        [\cmvec{{\cmmom_{\#_i}}}{(\dindx)} {\widehat{\times}}_{6\dindx}^{\ast}] & \zerom_{6(\dindx+1)\times 6}
    \end{bmatrix}.
\end{align}
\end{subequations}
\end{lemma}
\proof
This expression is obtained by taking the variation of the
relationship between momentum and force.
Specifically, the variation of \eqref{link_mom_to_force} is given as follows
\begin{align}
    \label{cm_force_momentum_var}
    \delta \cmvec{{\cmforce_{\#_i}}}{(\dindx)}
    &=
    \bm{\mathfrak{U}}(\cadjcm{\cmvec{{\cmv_{\link_i}}}{(\dindx)}{}})
    \delta \cmvec{{\cmmom_{\#_i}}}{(\dindx+1)}
    \notag
    \\
    &\quad\quad+
    \delta \bm{\mathfrak{U}}(\cadjcm{\cmvec{{\cmv_{\link_i}}}{(\dindx)}{}})
    \cmvec{{\cmmom_{\#_i}}}{(\dindx+1)}.
\end{align}
The term $\delta \bm{\mathfrak{U}}(\cadjcm{\cmvec{{\cmv_{\link_i}}}{(\dindx)}{}})$ in the second term of \eqref{cm_force_momentum_var} can be computed as follows.
\begin{align}
    \label{cnatm_var}
    &\delta \bm{\mathfrak{U}}(\cadjcm{\cmvec{{\cmv_{\link_i}}}{(\dindx)}{}})
    \cmvec{{\cmmom_{\#_i}}}{(\dindx+1)} 
    \notag\\
    &\quad\quad=
    \begin{bmatrix}
        \cadjcm{\delta \cmvec{{\cmv_{\link_i}}}{(\dindx)}{}} & \zerom_{6(\dindx+1)\times 6}
    \end{bmatrix}
    \cmvec{{\cmmom_{\#_i}}}{(\dindx+1)}
    \notag \\
    &\quad\quad=
    \begin{bmatrix}
        [\cmvec{{\cmmom_{\#_i}}}{(\dindx)} {\widehat{\times}}_{6\dindx}^{\ast}] & \zerom_{6(\dindx+1)\times 6}
    \end{bmatrix}
    \delta \cmvec{{\cmv_{\link_i}}}{(\dindx+1)} 
\end{align}
By substituting \eqref{cnatm_var} to \eqref{cm_force_momentum_var}, 
we obtain \eqref{var_cmmom_cmforce} on Lemma \ref{lem:mom_force}.

\begin{lemma}
    \label{lem:torque}
    The Jacobian mapping from $\delta \cmforce_{\joint_i}$ to $\delta \cmtorque_{\joint_i}$ is given by
    \begin{subequations}
    \begin{align}
        \label{var_cm_torque}
        \delta \cmtorque_{\joint_i}
        &=
        \cmgjacobarrow{\cmforce_{\joint_i}}{\cmtorque_{\joint_i}}
        \delta \cmforce_{\joint_i}, \\
        \quad&
        \cmgjacobarrow{\cmforce_{\joint_i}}{\cmtorque_{\joint_i}}
        =
        \cmslctm_{\joint_i}^{\trans}
    \end{align}  
    \end{subequations}
\end{lemma}
\proof
This result follows directly from the variation of \eqref{cm_torque_force}.

\begin{table}[tbp]
\centering
\caption{
Relationship between each lemma, the variation of the equations leading to Section~\ref{sec:cmcomp}, and the elemental relation in Section~\ref{sec:elem_grad}.
}
\resizebox{\linewidth}{!}{
\begin{tabular}{l l c c}
\toprule
Lemma & Description & Variation (Sec.~\ref{sec:cmcomp}) & Relation (Sec.~\ref{sec:elem_grad}) \\
\midrule

Lemma~\ref{lem:j_vel} 
& Joint velocity mapping 
& \eqref{joint_spatial_cm_trans} 
& (direct mapping) \\

\midrule

\multirow{3}{*}{Lemma~\ref{lem:cm_vel}} 
& Step 1: Tangent mapping 
& \multirow{3}{*}{\eqref{cm_dkine}} 
& \eqref{cm_tan_map_2} \\
& Step 2: Kinematic propagation 
&
& \eqref{tan_dfk} and \eqref{recursive_close_form} \\
& Step 3: Velocity recovery 
& 
& \eqref{cm_tan_map}  \\

\midrule

Lemma~\ref{lem:lmom} 
& Inertia mapping 
& \eqref{def_cmtm_moment} 
& (direct mapping) \\

\midrule

\multirow{3}{*}{Lemma~\ref{lem:mom}} 
& Step 1: World transformation 
& \multirow{3}{*}{\eqref{cm_moment_dynamics}} 
& \eqref{eq:w_cm_mom} and \eqref{cmtm_var_arb_vec} \\
& Step 2: Momentum propagation 
& 
& \eqref{wcm_moment_dynamics_var} and \eqref{recursive_close_form}\\
& Step 3: Local recovery 
& 
& \eqref{eq:w_cm_mom}, \eqref{cmtm_var_arb_vec} and \eqref{var_inv_trans_vec}\\

\midrule

Lemma~\ref{lem:mom_force} 
& Force mapping 
& \eqref{cm_force_momentum} 
& \eqref{link_mom_to_force}\\

\midrule

Lemma~\ref{lem:torque} 
& Torque extraction 
& \eqref{cm_torque_force}
& (direct mapping)\\

\bottomrule
\end{tabular}
}
\label{tab:lemma_structure}
\end{table}

\subsection{Jacobian construction via composition}
The Jacobians are constructed in a hierarchical manner, following the dependency structure of physical quantities.

\paragraph{Joint velocity Jacobian.}
The comprehensive joint velocity can already be derived from Lemma~\ref{lem:j_vel}.

The Jacobian associated with the comprehensive link spatial velocity is obtained by composing the mappings in Lemma~\ref{lem:cm_vel}.
\paragraph{Joint tangent vector Jacobian.}
\textit{(Construction of $\cmtanjacobj{i}$)}
First, the tangent vector at the joint is expressed as
\begin{align}
    \label{cm_tan_joint_jacob}
    \cmvec{{\cmtanv_{\joint_i}}}{(\dindx+1)}
    = \cmtanjacobj{i} \delta \cmvec{{{\cmjang}_{\joint_i}}}{(\dindx)},
\end{align}
where $\cmtanjacobj{i}$ is defined by 
$$
    \cmtanjacobj{i} := \cmgjacobarrow{\cmptanv_{\joint_i}}{\cmtanv_{\joint_i}}
    \cmgjacobarrow{\cmjang_{\joint_i}}{\cmptanv_{\joint_i}}.
$$
This mapping is obtained by applying the expression in \eqref{var_cmv_1} of Lemma~\ref{lem:cm_vel} to the mapping toward $\cmptanv_{\joint_i}$ defined in \eqref{joint_spatial_cmtan_trans}.

\paragraph{Joint tangent vector Jacobian.}
\textit{(Construction of $\cmtanjacob{}{}$)}
Next, the tangent vector is propagated to the link as
\begin{align}
    \label{var_tan_cmtm_link_jacob}
    \cmvec{{\cmtanv_{\link_i}}}{(\dindx+1)}
    = \cmtanjacob{i}{} 
    \delta \cmvec{{\cmjang_{\joint}}}{(\dindx)}
    = \sum_{j \in \parentss{i}} 
    \cmtanjacob{i}{j} 
    \delta \cmvec{{{\cmjang}_{\joint_j}}}{(\dindx)}
\end{align}
where $\cmtanjacob{i}{j}$ is defined by 
\begin{align}
    \cmtanjacob{i}{j} &:=
    \cmgjacobarrow{\cmtanv_{\joint_i}}{\cmtanv_{\link_i}}\cmtanjacobj{j} 
\end{align}
This mapping is obtained by applying the expression in \eqref{var_cmv_2} of Lemma~\ref{lem:cm_vel} to the mapping toward $\cmtanv_{\joint_i}$ defined in \eqref{cm_tan_joint_jacob}.

\paragraph{Joint spatial velocity Jacobian.}
\textit{(Construction of $\cmjacob{}$)}
Finally, the variation of the spatial velocity is recovered as
\begin{align}
    \label{cm_vel_link_jacob}
    \delta  \cmvec{{\cmv_{\link_i}}}{(\dindx)}
    = \cmjacob{i} \delta \cmvec{{\cmjang_\joint}}{(\dindx)},
\end{align}
where $\cmjacob{i}$ is defined by 
$$
    \cmjacob{i} := \cmgjacobarrow{\cmtanv_{\link_i}}{\cmv_{\link_i}}  \cmtanjacob{i}.
$$
This mapping is obtained by applying the expression in \eqref{var_cmv_3} of Lemma~\ref{lem:cm_vel} to the mapping toward $\cmtanv_{\link_i}$ defined in \eqref{var_tan_cmtm_link_jacob}.

The comprehensive spatial velocity $\cmv_{\link_i}$ can be expressed as a function of the generalized coordinates via \eqref{cmtm_fk}, i.e.,
$\cmv_{\link_i} = \bm{f}(\cmjang)$,
and its Jacobian is given by
$\cmjacob{i} = \frac{\partial \bm{f}}{\partial \cmjang}$.
On the other hand, the matrix $\cmjacobdt$ in \eqref{cm_basic_jacob} does not coincide with this Jacobian, i.e.,
$\cmjacobdt \neq \frac{\partial \bm{f}(\cmjang)}{\partial \cmjang}$.
However, it satisfies
$\cmjacobdt \frac{\partial \cmjang}{\partial t}
= \cmjacob{i} \frac{\partial \cmjang}{\partial t}$.
Despite this equality, $\cmjacobdt$ is not the true derivative of $\bm{f}(\cmjang)$ and therefore cannot be used for optimization.

\paragraph{Link momentum Jacobian.}
\textit{(Construction of $\cmlmomjacob{}$)}
We next construct the Jacobian associated with momentum by applying the inertia mapping to the velocity Jacobian.
By applying Lemma~\ref{lem:lmom}, the Jacobian associated with the comprehensive link momentum is obtained by composing the inertia mapping with the Jacobian of spatial velocity:
\begin{align}
    \label{link_mom_jacob}
    \delta \cmmom_{\link_i} 
    = \cmlmomjacob{i} \, \delta \cmjang_\joint,
\end{align}
where $\cmlmomjacob{i}$ is defined by 
$$
    \cmlmomjacob{i} := \cmgjacobarrow{\cmv_{\link_i}}{\cmmom_{\link_i} } \cmjacob{i}.
$$
This expression shows that the Jacobian of momentum is constructed by applying a linear transformation to the Jacobian of spatial velocity, preserving the underlying structured propagation.

\paragraph{Link world momentum Jacobian.}
\textit{(Construction of $\wcmlmomjacob{}$)}
We then extend the construction to the world-frame momentum by incorporating the transformation and propagation mappings in Lemma~\ref{lem:mom}.
Specifically, by substituting the link momentum Jacobian \eqref{link_mom_jacob} and the tangent propagation \eqref{var_tan_cmtm_link_jacob} into the world frame transformation \eqref{var_link_world_mom_local_mom} on Lemma \ref{lem:mom}, we obtain
\begin{align}
\label{link_world_mom_jacob}
\delta \wcmmom_{\link_i} 
&= 
\wcmlmomjacob{i} \delta \cmjang_{\joint},
\end{align}
where $\wcmlmomjacob{i}$ is defined by 
$$
\wcmlmomjacob{i}
:= 
\cmgjacobarrow{\cmmom_{\link_i}}{\wcmmom_{\link_i} } \cmlmomjacob{i}
+
\cmgjacobarrow{\cmtanv_{\link_i}}{\wcmmom_{\link_i} } 
\cmtanjacob{i}{}
.
$$
This expression shows that the Jacobian is obtained by combining two components: the transformed momentum Jacobian and the correction term arising from the variation of the transformation.

\paragraph{Joint world momentum Jacobian.}
\textit{(Construction of $\wcmjmomjacob{i}{}$)}
By applying the closed-form expansion of the recursive relation in \eqref{eq:structure_main}, as expressed in \eqref{eq:structure}, to the world-frame momentum propagation \eqref{lem_var_wmom} on Lemma \ref{lem:mom}, we obtain the accumulated world-frame joint momentum variation as
\begin{align}
    \label{joint_world_mom_jacob}
    \delta \wcmmom_{\joint_i} 
    =
    \wcmjmomjacob{i}{{}}  \delta \cmjang_\joint,
\end{align}
where $\wcmjmomjacob{i}{{}}$ is defined by 
$$
\wcmjmomjacob{i}{{}} 
:= 
\sum_{j \in \childrens{i}} 
\cmgjacobarrow{\wcmmom_{\link_j}}{\wcmmom_{\joint_i}}
\wcmlmomjacob{j}.
$$

\paragraph{Joint momentum Jacobian.}
\textit{(Construction of $\cmjmomjacob{}$)}
Finally, the Jacobian of joint momentum is obtained by transforming the accumulated world-frame momentum back to the local frame.
By substituting the world joint momentum Jacobian \eqref{joint_world_mom_jacob} and the tangent propagation \eqref{var_tan_cmtm_link_jacob} into the local frame transformation \eqref{j_cmt_mom_var} on Lemma \ref{lem:mom}, we obtain
\begin{align}
    \label{joint_mom_jacob}
    \delta \cmmom_{\joint_i} 
    = \cmjmomjacob{i} \delta \cmjang_{\joint},
\end{align}
where $\cmjmomjacob{i}$ is defined by 
$$
    \cmjmomjacob{i} := 
    \cmgjacobarrow{\wcmmom_{\joint_i}}{\cmmom_{\joint_i}}
    \wcmjmomjacob{i}{{}}
    +
    \cmgjacobarrow{\cmtanv_{\link_i}}{\cmmom_{\joint_i}}
    \cmtanjacob{i}.
$$

\paragraph{Link force Jacobian.}
\textit{(Construction of $\cmlforcejacob{}$)}
The Jacobian of force is constructed by combining the momentum and velocity Jacobians through the mapping defined in Lemma~\ref{lem:mom_force}.
By substituting the link momentum Jacobian \eqref{link_mom_jacob} and the link spatial velocity Jacobian \eqref{cm_vel_link_jacob} into the world frame transformation \eqref{var_cmmom_cmforce} on Lemma \ref{lem:mom_force}, we obtain
\begin{align}
    \label{cm_link_force_jacob}
    \delta \cmvec{{\cmforce_{\link_i}}}{(\dindx)}
    = 
    \cmlforcejacob{i} 
    \delta \cmvec{{\cmjang_\joint}}{(\dindx+1)},
\end{align}
where $\cmlforcejacob{i}$ is defined by 
$$
\cmlforcejacob{i} := 
\cmgjacobarrow{\cmmom_{\link_i}}{\cmforce_{\link_i}}
{\cmlmomjacob{i}} + 
\cmgjacobarrow{\cmv_{\link_i}}{\cmforce_{\link_i}}
{\cmjacob{i}}.
$$
\paragraph{Joint force Jacobian.}
\textit{(Construction of $\cmjforcejacob{}$)}
Similarly, by applying the joint moment Jacobian into \eqref{var_cmmom_cmforce}, we obtain 
\begin{align}
    \label{cm_joint_force_jacob}
    \delta \cmvec{{\cmforce_{\joint_i}}}{(\dindx)} &= \cmjforcejacob{i} 
    \delta \cmvec{{\cmjang_\joint}}{(\dindx+1)}, 
\end{align}
where $\cmjforcejacob{i}$ is defined by 
$$
    \cmjforcejacob{i} :=     
    \cmgjacobarrow{\cmmom_{\joint_i}}{\cmforce_{\joint_i}}
    \cmjmomjacob{i} + 
    \cmgjacobarrow{\cmv_{\joint_i}}{\cmforce_{\joint_i}}
    \cmjacob{i}.
$$

\paragraph{Joint torque Jacobian.}
\textit{(Construction of $\cmtorquejacob{}$)}
Finally, the Jacobian of torque is obtained by applying the selection mapping to the joint force Jacobian.
By substituting \eqref{cm_joint_force_jacob} into \eqref{var_cm_torque} on Lemma \ref{lem:torque},
we obtain
\begin{align}
\label{torque_jacob}
    \delta \cmvec{{\cmtorque_{\joint_i}}}{(\dindx)}  = \cmtorquejacob{i} 
    \delta \cmvec{{\cmjang_\joint}}{(\dindx+1)}, 
\end{align}
where $\cmtorquejacob{i}$ is defined by 
$$
    \cmtorquejacob{i} := 
    \cmgjacobarrow{\cmforce_{\joint_i}}{\cmtorque_{\joint_i}}
    \cmjforcejacob{i}.
$$
\section{Numerical verification}
\subsection{Experimental setup}

\subsubsection{Robot models}
Two robot models are used in the experiments.
The first model is a simple serial-link robot arm used to evaluate scalability with respect to the number of degrees of freedom.
Each link has a length of $1\,\mathrm{m}$, a mass of $5\,\mathrm{kg}$, a center of mass located at the geometric center of the link, and a spatial inertia vector given by
a rotational inertia tensor represented by
$\begin{bmatrix}
    I_{xx} & I_{yy} & I_{zz} & I_{xy} & I_{xz} & I_{yz}
\end{bmatrix}
= \begin{bmatrix}0.1&0.1&0.1&0&0&0\end{bmatrix}$.
The number of links is increased to generate robots with different degrees of freedom.

The second model is a 7-DoF robot arm based on the \cite{franka} model.
The robot has a total weight of approximately $18\,\mathrm{kg}$ and a maximum reach of $855\,\mathrm{mm}$.

Figure~\ref{fig:overview} (right) shows the 3-DoF version of the simple robot model and the 7-DoF robot arm used in the experiments.

\subsubsection{Trajectory optimization}
Joint trajectories are parameterized using B-splines in order to ensure smooth representations of higher-order time-derivative quantities while maintaining a compact parametrization \cite{ayusawa2019predictive,suleiman2010time}.
The optimization variables are defined as the control points of the spline representation.

By stacking all control points into a parameter vector $\bm{\theta}_{\joint_i}$, the joint trajectory and its higher-order time derivatives can be represented as linear functions of $\bm{\theta}_{\joint_i}$:
\begin{align}
\parvec{\bm{q}_{\joint_i}(t)}{k}{t}
=
\parvec{\bm{D}_p(t)}{k}{t}
\bm{\theta}_{\joint_i}.
\end{align}

The recursive definition of the B-spline basis functions and their higher-order derivatives are summarized in
\eqref{eq:D_def}  Appendix~\ref{app:bspline}.

In practical trajectory optimization problems, the optimization is performed under equality and inequality constraints associated with the robot motion and physical feasibility.
In this study, the joint trajectories were parameterized using the B-spline representation introduced above, and the optimization variables were defined as the spline parameter vector $\bm{\theta}_{\joint}$, where the subscript $\joint$ denotes the stacked representation associated with all joints.
Accordingly, the optimization problem introduced in \eqref{opt_problem} was solved as the following constrained nonlinear problem:
\begin{align}
    \min_{\bm{\theta}_{\joint}}
        f(\bm{\theta}_{\joint})
    \quad
    \mathrm{s.t.}\quad
    \bm{g}(\bm{\theta}_{\joint}) = \bm{0},
    \quad
    \bm{h}(\bm{\theta}_{\joint}) \ge \bm{0},
\end{align}
where $\bm{f}(\bm{\theta}_{\joint})$ denotes the stacked weighted residual vector corresponding to the trajectory cost terms, while $\bm{g}(\bm{\theta}_{\joint})$ and $\bm{h}(\bm{\theta}_{\joint})$ represent equality and inequality constraints, respectively.

The constraints were handled using a penalty method, where constraint violations were incorporated into the objective function as additional penalty terms with weights approximately $10^{12}$ times larger than those of the nominal trajectory cost terms.
As a result, the constrained optimization problem was approximately transformed into an unconstrained nonlinear least-squares problem while strongly discouraging constraint violations.
The resulting optimization problem was solved using a damped Gauss--Newton method with backtracking line search.
A fixed damping coefficient of $\lambda = 10^{-8}$ was used throughout the optimization.
The optimization was terminated when the norm of the optimization value update became smaller than $10^{-12}$, or when the number of iterations reached the maximum iteration count of $500$.

\subsubsection{Computation environment}

All experiments were conducted on a laptop equipped with an AMD Ryzen 9 7940HS CPU (8 cores, 16 threads, up to 5.2 GHz) and 32 GB RAM running Ubuntu 24.04 LTS.

\begin{table}[tbp]
    \centering
    \caption{Maximum absolute differences of each element between the Jacobians computed by numerical differentiation and those obtained analytically, evaluated up to the fourth derivative in 7-DOF manipulator.}
    \resizebox{\linewidth}{!}{
    \begin{tabular}{c|c c c c}
         & $k = 0$ & $k = 1$ & $k = 3$ & $k = 4$ \\
        \hline
        $\cmjacob{i}$  & $8.99 \times 10^{-9}$ & $1.34 \times 10^{-8}$ & $1.33 \times 10^{-8}$ & $1.48 \times 10^{-8}$ \\
        $\cmlmomjacob{i}$ & $1.23 \times 10^{-8}$ & $1.91 \times 10^{-8}$ & $2.71 \times 10^{-8}$ & $1.74 \times 10^{-7}$ \\
        $\wcmlmomjacob{i}$ & $1.72 \times 10^{-8}$ & $1.91 \times 10^{-7}$ & $3.67 \times 10^{-7}$ & $1.30 \times 10^{-7}$ \\
        $\wcmjmomjacob{i}{}$ & $4.06 \times 10^{-8}$ & $6.14 \times 10^{-8}$ & $7.53 \times 10^{-8}$ & $1.49 \times 10^{-7}$ \\
        $\cmjmomjacob{i}{}$ & $3.88 \times 10^{-8}$ & $5.76 \times 10^{-8}$ & $7.30 \times 10^{-8}$ & $2.11 \times 10^{-7}$ \\
        $\cmlforcejacob{i}$ & $1.61 \times 10^{-8}$ & $4.26 \times 10^{-8}$ & $2.75 \times 10^{-7}$ & $2.72 \times 10^{-7}$ \\
        $\cmjforcejacob{i}$ & $6.94 \times 10^{-8}$ & $7.34 \times 10^{-8}$ & $2.75 \times 10^{-7}$ & $2.63 \times 10^{-7}$ \\
        $\cmtorquejacob{i}$ & $7.55 \times 10^{-8}$ & $1.28 \times 10^{-7}$ & $2.04 \times 10^{-7}$ & $4.77 \times 10^{-7}$
    \end{tabular}
    }
    \label{tab:jacob_error}
\end{table}

\begin{figure}[tbp]
    \centering
    \includegraphics[width=\linewidth]{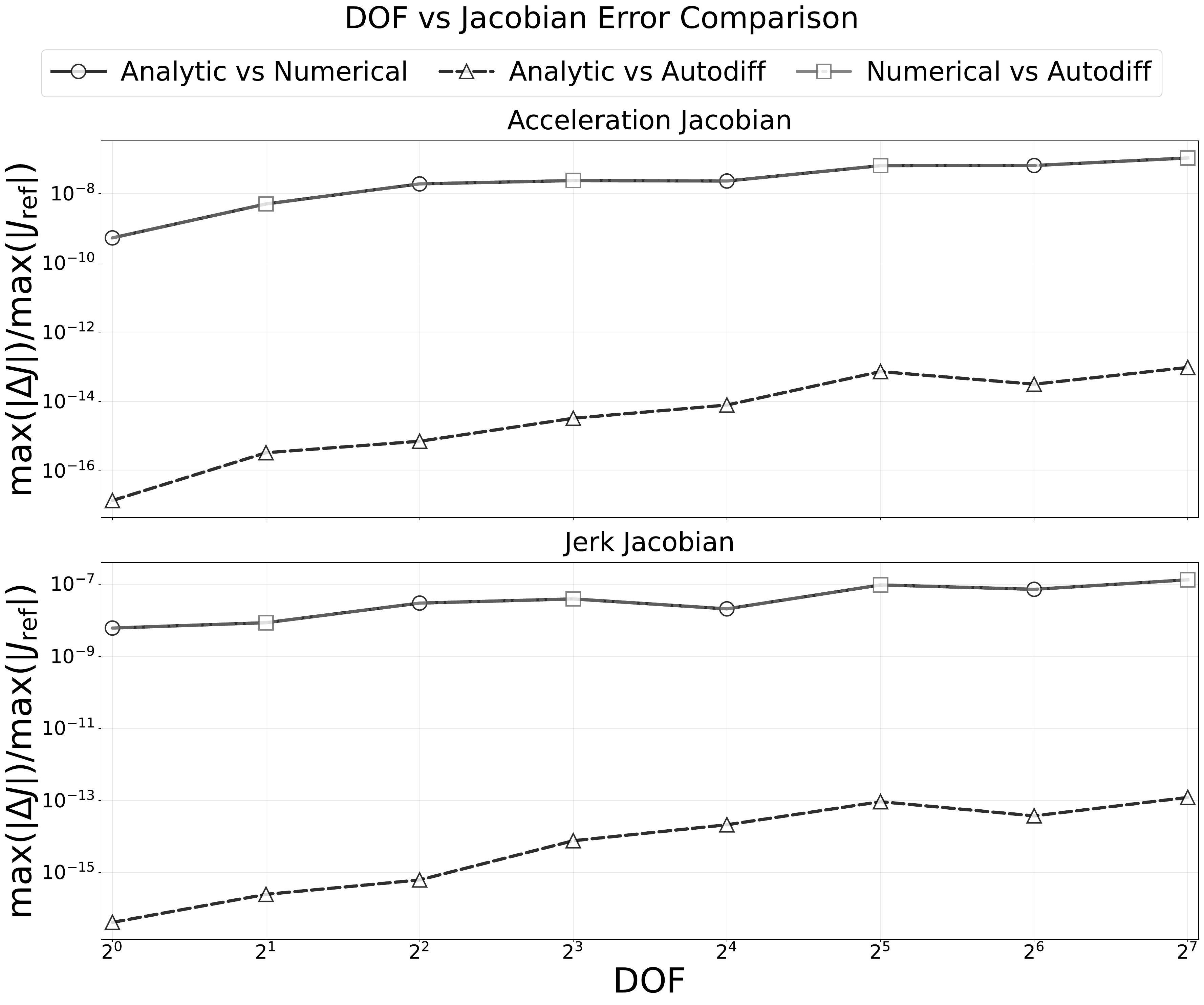}
    \caption{Comparison of gradient computation proposed analytical method, numerical difference and automatic difference.}
    \label{fig:comparison_gradient}
\end{figure}

\begin{figure*}[h]
    \centering
    \begin{subfigure}[t]{0.48\linewidth}
        \centering
        \includegraphics[width=\linewidth]{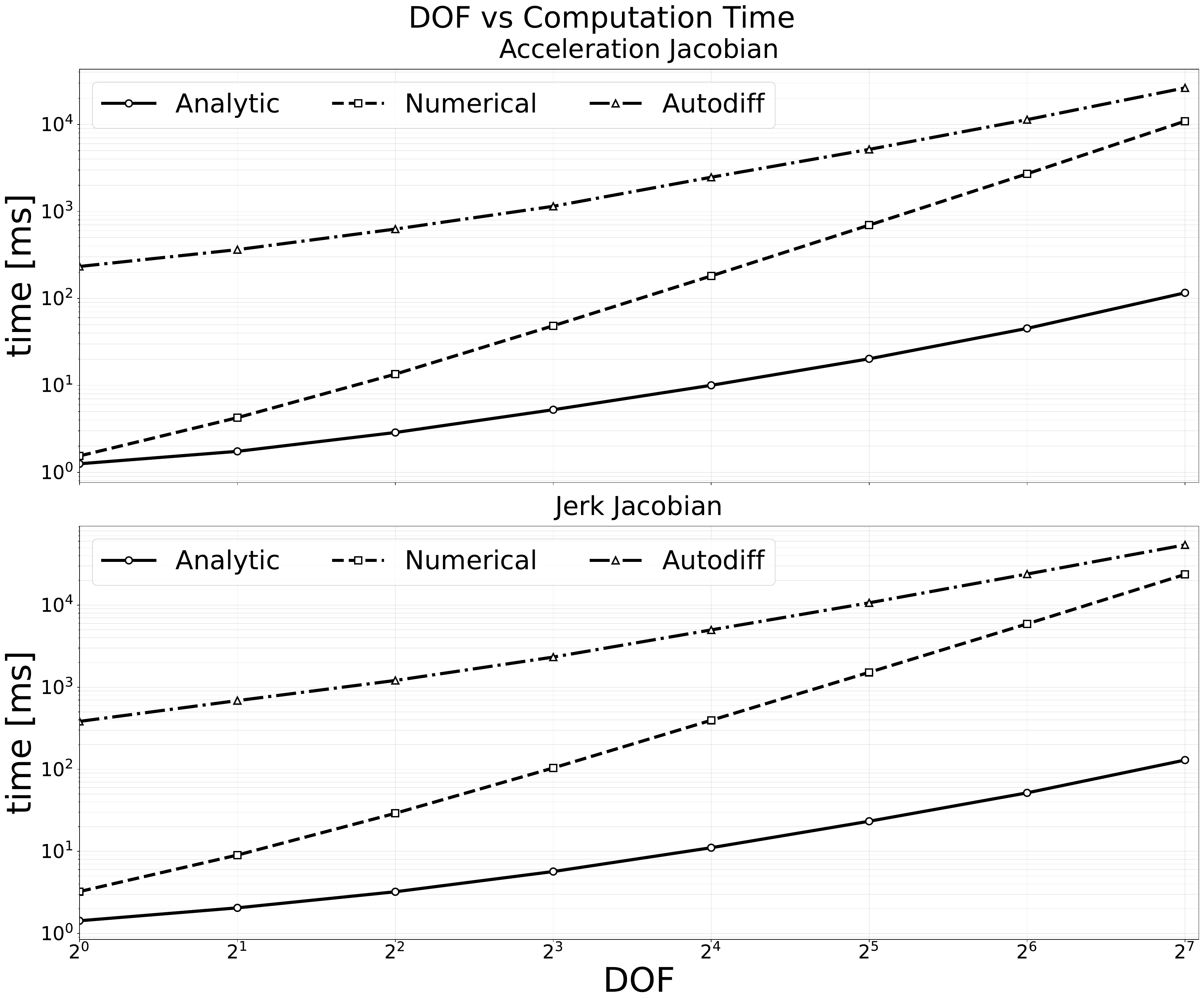}
        \caption{Computation time comparison of Jacobian matrices. Analytical, numerical, and automatic differentiation methods are compared.}
        \label{fig:comparison_gradient_time}
    \end{subfigure}
    \hfill
    \begin{subfigure}[t]{0.48\linewidth}
        \centering
        \includegraphics[width=\linewidth]{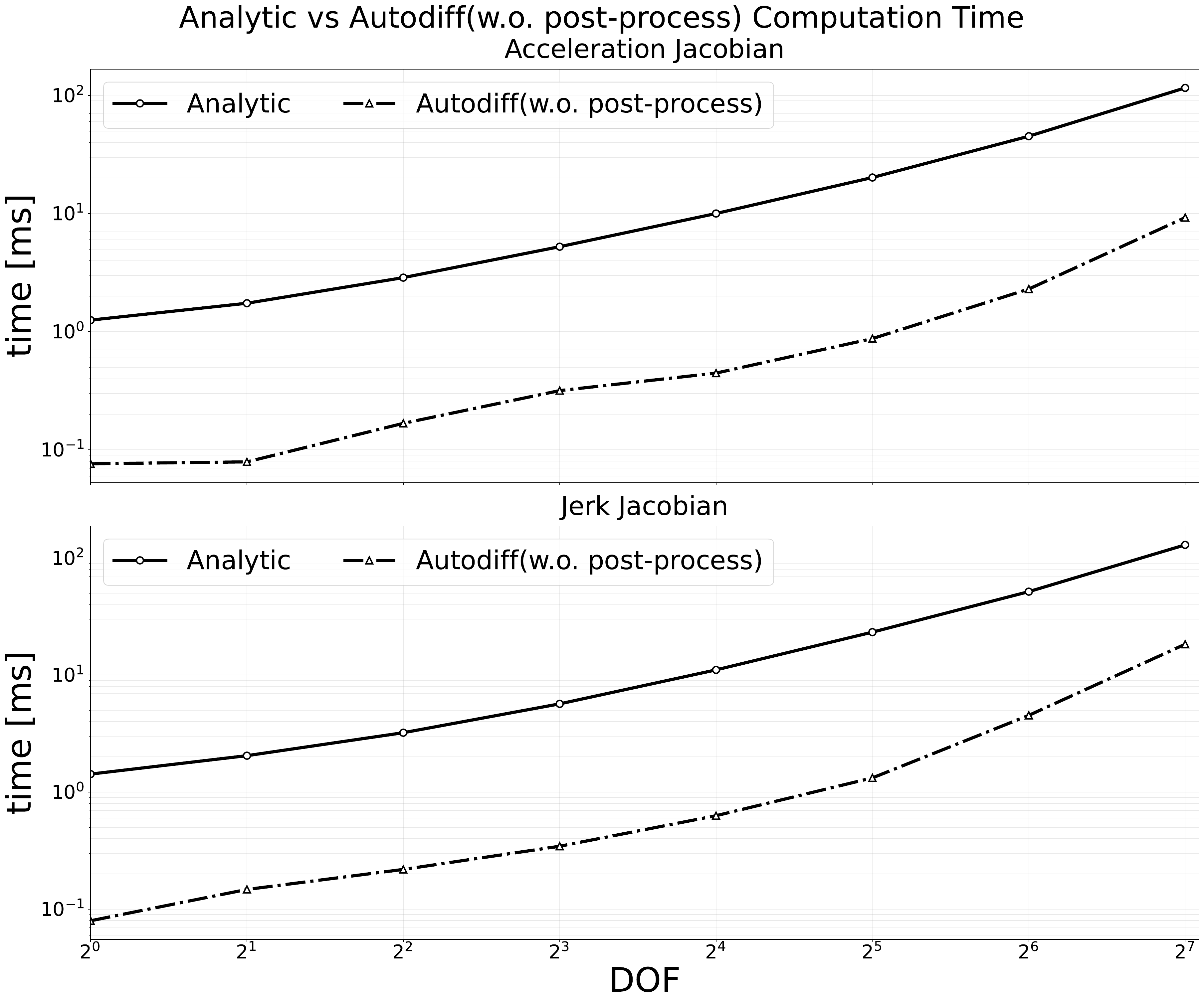}
        \caption{Evaluation time comparison after post-processing. Analytical and automatic differentiation exhibit similar computational scaling.}
        \label{fig:jacobian_eval_scaling_post}
    \end{subfigure}
    \caption{
    Computational performance of Jacobian evaluation methods.
    (Left) Overall computation time comparison including numerical differentiation.
    (Right) Evaluation time comparison between analytical and automatic differentiation after post-processing.
    }
    \label{fig:jacobian_time_comparison}
\end{figure*}

\begin{figure*}
    \centering
    \includegraphics[width=\linewidth]{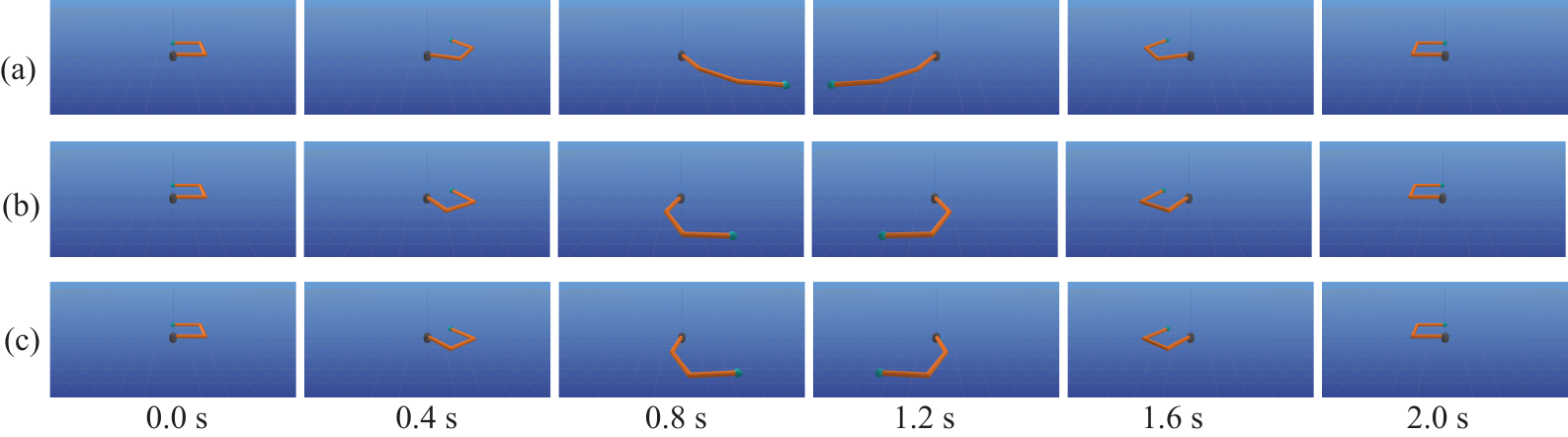}
    \caption{Snapshot of simulation in case (a), (b), (c)}
    \label{fig:3link_robot_sim}
\end{figure*}

\begin{table*}[t]
\centering
\caption{
Comparison between the true cost weights and those estimated via inverse optimal control for different systems and cost settings.
}
\resizebox{\linewidth}{!}{
\begin{tabular}{c|c|l|c|c|c}
\hline
System & Case & Cost & True cost weights $\bm{w}$ & Estimated cost weights $\bm{w}_{\text{est}}$ & $L_1$ error $\|\bm{w} - \bm{w}_{\mathrm{est}}\|_1$ \\
\hline
3-DoF & (a) & $\spv, \dot{\spv}, \ddot{\spv}$ & [0.333333, 0.333333, 0.333333] 
      & [0.333333, 0.333333, 0.333333] 
      & $2.56 \times 10^{-8}$ \\
\hline
3-DoF & (b) & $\spv, \dot{\spv}, \ddot{\spv}, \torque$& [0.25, 0.25, 0.25, 0.25] 
      & [0.25, 0.25, 0.25, 0.25] 
      & $6.87 \times 10^{-7}$ \\
\hline
3-DoF & (c) & $\spv, \dot{\spv}, \ddot{\spv}, \torque, \dot{\torque}$ & [0.2, 0.2, 0.2, 0.2, 0.2] 
      & [0.199998, 0.200001, 0.2, 0.2, 0.2] 
      & $3.89 \times 10^{-6}$ \\
\hline
7-DoF & -- & $\spv, \dot{\spv}, \ddot{\spv}, \torque$ & [0.332226, 0.332226, 0.332226, 0.003322] 
      & [0.332226, 0.332226, 0.332226, 0.003322] 
      &$1.07 \times 10^{-7}$\\
\hline
\end{tabular}
}
\label{tab:ioc_weights}
\end{table*}

\subsection{Comparative analysis of Jacobian calculation}



To evaluate the proposed Jacobian computation method, we compare it with numerical differentiation and automatic differentiation. 
We consider a serial-link robot and compute the Jacobians with respect to the generalized coordinates.

First, we evaluate the accuracy of the proposed analytical Jacobians.
The discrepancy between two Jacobians is evaluated using the normalized maximum absolute element-wise difference to account for the scale of the Jacobian entries,
\begin{align}
    e_J =
    \frac{
        \max_{i,j}
        \left|
            \left(J_{\mathrm{proposed}}\right)_{ij}
            -
            \left(J_{\mathrm{ref}}\right)_{ij}
        \right|
    }{
        \max_{i,j}
        \left|
            \left(J_{\mathrm{ref}}\right)_{ij}
        \right|
    },
\end{align}
where $J_{\mathrm{ref}}$ denotes the Jacobian obtained by either numerical differentiation or automatic differentiation.

Figure~\ref{fig:comparison_gradient} evaluates the accuracy of the proposed method by comparing the norms of Jacobians associated with the variations of velocity and acceleration.
We consider simple serial-link robots and increase the number of degrees of freedom up to $2^7 = 128$.
The horizontal axis represents the number of degrees of freedom, while the vertical axis shows the discrepancy between Jacobians.
The proposed analytical Jacobians closely match those obtained by automatic differentiation across all tested configurations.
For a 128-DoF system, the discrepancy is $9.61 \times 10^{-14}$ for velocity variations and $1.21 \times 10^{-13}$ for acceleration variations.
In contrast, when compared with numerical differentiation, the discrepancies are significantly larger.
At 128 DoF, the errors are $1.07 \times 10^{-7}$ for velocity variations and $1.33 \times 10^{-7}$ for acceleration variations, reflecting the approximation errors inherent in numerical differentiation.
We also evaluated the discrepancy between numerical differentiation and automatic differentiation.
This discrepancy exhibits nearly identical values to those observed between numerical differentiation and the proposed analytical method, resulting in overlapping curves in the figure.
This indicates that the difference between the analytical and automatic differentiation results is negligible, and that the dominant source of error arises from numerical differentiation.

To further assess the accuracy for higher-order derivatives, we conduct experiments on a 7-DoF robot manipulator (\cite{franka}) and compare the analytical Jacobians with those obtained by numerical differentiation, as summarized in Table~\ref{tab:jacob_error}.
In this setting, automatic differentiation for higher-order derivatives is not employed, and numerical differentiation is used as a reference for comparison.
As the derivative order and the complexity of the underlying computations increase, the discrepancies become more pronounced.
Nevertheless, the discrepancies remain within the order of $10^{-7}$ even for higher-order derivatives, suggesting that the analytical Jacobians are computed consistently.
The increase in discrepancy is likely related to the accumulation of numerical errors in the comparison process.

Next, we evaluate the computational efficiency. 
Figure~\ref{fig:comparison_gradient_time} shows both the norm of the difference between gradient matrices and the computation time as functions of the number of robot degrees of freedom. 
The proposed analytical Jacobian achieves significantly faster computation than numerical differentiation and is also competitive with automatic differentiation. 
This is because automatic differentiation involves a non-negligible construction cost for building the computational graph, in addition to the evaluation cost. 
For example, for a system with 128 degrees of freedom, the analytical method requires on the order of $10^2$ ms, whereas numerical differentiation requires on the order of $10^4$ ms.

After post-processing, automatic differentiation provides faster evaluation than the proposed analytical Jacobian. 
However, both methods exhibit similar computational scaling, as evidenced by the parallel trends in Figure~\ref{fig:jacobian_eval_scaling_post}.



\subsection{Verification of direct and inverse optimal calculations through simulation}

We validate the proposed higher-order Jacobian formulation through robot motion optimization based on the direct and inverse optimization framework described in Section~\ref{sec:dioc}.

First, we consider a 3-DoF simple robot arm and perform direct motion optimization under different cost functions.
The motion is generated over a trajectory duration of $2.0$ s, discretized into $200$ samples with a time step of $0.01$ s.
The joint trajectories are represented using fifth-order B-splines ($p = 5$) with $n_c = 50$ control points, providing sufficient smoothness for computing higher-order motion quantities such as jerk and torque rates.
All joint angles are constrained to move from $1.57$ rad to $-1.57$ rad, while the joint velocities at both the initial and terminal times are fixed to zero.
In addition, joint limit constraints were imposed as inequality constraints for all joints within the range $[-\pi, \pi]$.

The optimization is performed under different combinations of motion-related cost terms introduced in \eqref{opt_problem}, including joint velocity, acceleration, jerk, torque, and torque-rate costs.
We consider three cost function settings:
(a) a cost function including joint velocity, acceleration, and jerk;
(b) the cost function in (a) with an additional joint torque term; and
(c) the cost function in (b) further augmented with the first derivative of joint torque.
The optimization is carried out over joint trajectories, and the corresponding joint torques are computed from the resulting trajectories using inverse dynamics. Based on the optimized trajectories, we then estimate the cost weights by solving the inverse problem formulated in \eqref{ikkt}, where the weight vector $\bm{w}$ is required to satisfy the KKT condition $\bm{w}^{\trans}\bm{C} = \bm{0}$.

Figure~\ref{fig:cost_vertical} shows the optimized joint angle trajectories along with the corresponding joint torques computed from them. Compared to Figure~\ref{fig:cost_vertical} (a), Figure~\ref{fig:cost_vertical} (b) demonstrates that including joint torque in the cost function reduces the magnitude of joint torques. Furthermore, Figure~\ref{fig:cost_vertical} (c) shows that incorporating the derivative of joint torque leads to smoother torque profiles with reduced variation.
Table~\ref{tab:ioc_weights} summarizes the comparison between the true cost weights and those estimated via the inverse formulation in \eqref{ikkt}. The results show that the estimated weights closely match the ground-truth values, with errors on the order of $10^{-8}$ to $10^{-6}$. Although the estimation error slightly increases as the number of cost terms grows, the overall accuracy remains high, demonstrating that the proposed method can reliably recover the underlying cost weights.
Simulation snapshots for cases (a)–(c) are shown in Figure~\ref{fig:3link_robot_sim}.

To further validate the proposed framework, we conducted a simulation study on a 7-DoF robot manipulator (\cite{franka}).
In this experiment, direct optimization was performed by defining a cost function composed of multiple physical quantities, including joint velocity, acceleration, jerk, and torque.
As in the 3-DoF case, the joint trajectories are parameterized using B-spline curves, and the optimization is carried out with respect to the corresponding control points.
The same boundary conditions are imposed at $t=0.0$~s and $t=2.0$~s, where the joint angles are specified from $1.57$ to $-1.57$ and the joint velocities are set to zero at both endpoints.
The optimization yields a joint trajectory that satisfies these criteria, and the resulting joint angles and joint torques over time are shown in Figure~\ref{fig:7dof_robot_doc_result}.
The generated motion exhibits smooth transitions in both kinematic and dynamic quantities, reflecting the influence of higher-order derivative terms in the cost function. 
As in the 3-link simple robot case, assigning identical weights to all cost terms caused the torque-related cost to dominate, leading to unstable optimization.
To address this, the weights were scaled by a factor of 1/100, based on the approximate magnitude of the robot weight (mass times gravitational acceleration).
To evaluate the consistency of the proposed framework in inverse optimization, we applied the inverse optimization method to the generated trajectory. The estimated cost weights are summarized in Table~\ref{tab:ioc_weights}, together with the ground-truth weights used in the direct optimization.
The results show that the estimated weights closely match the original ones, with an estimation error of $1.07 \times 10^{-7}$, demonstrating that the proposed structured gradient computation enables reliable recovery of cost function parameters even when higher-order time derivatives are included.
The resulting robot motion obtained from the optimization is shown in Figure~\ref{fig:7link_robotsim}.

These results collectively validate that the proposed framework enables direct motion optimization involving higher-order time-derivative variables.
The experiments also demonstrate that variations in cost functions defined on higher-order physical quantities lead to distinct changes in the resulting motion.
Furthermore, the inverse optimal control results confirm that the underlying cost weights can be accurately recovered from the optimized trajectories.

\begin{figure}[tbp]
\centering

\begin{subfigure}[b]{\columnwidth}
  \centering
  \includegraphics[width=\linewidth]{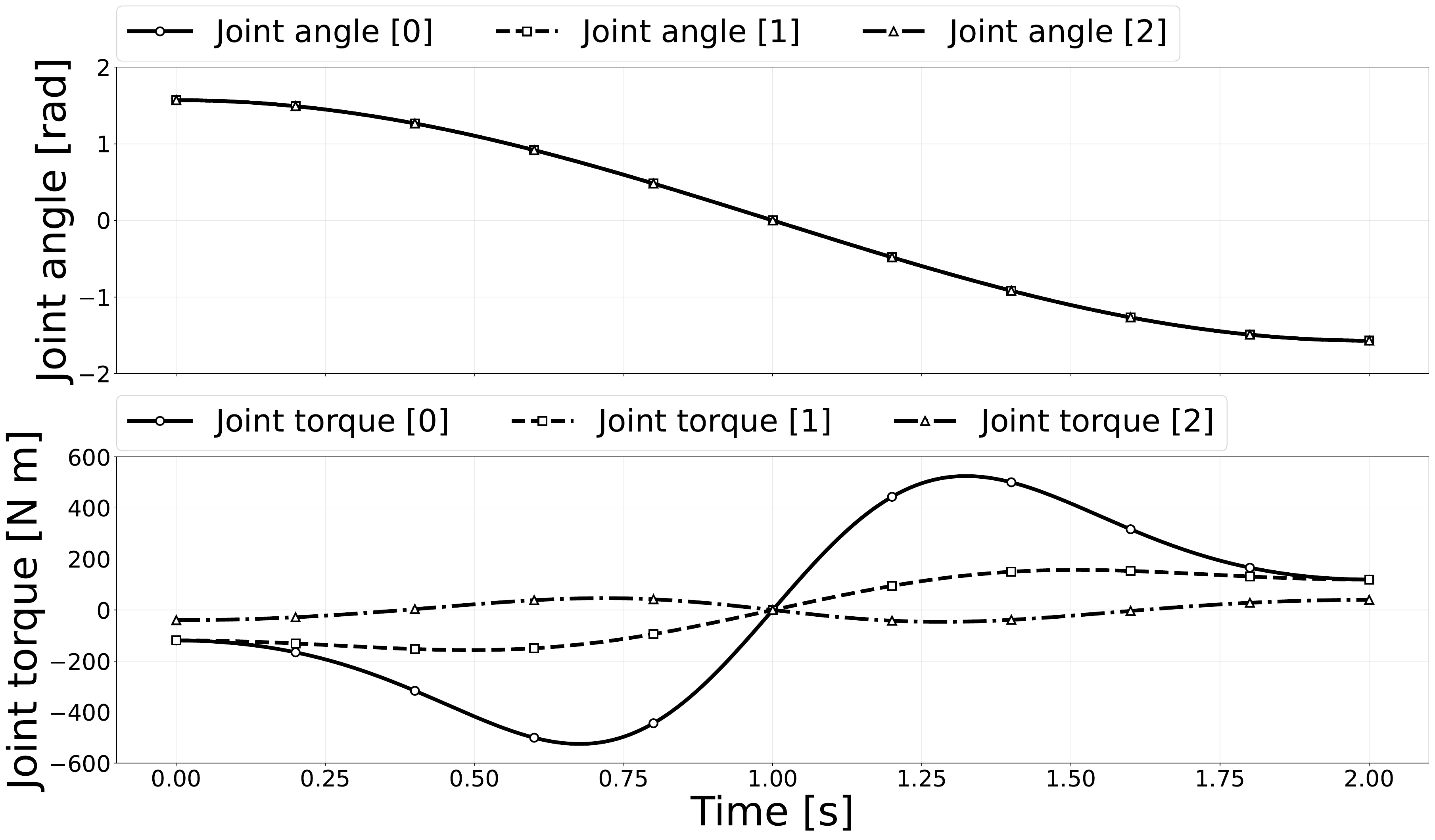}
  \caption{Condition Minimization of joint velocity, acceleration, and jerk.}
\end{subfigure}

\vspace{0.5em}

\begin{subfigure}[b]{\columnwidth}
  \centering
  \includegraphics[width=\linewidth]{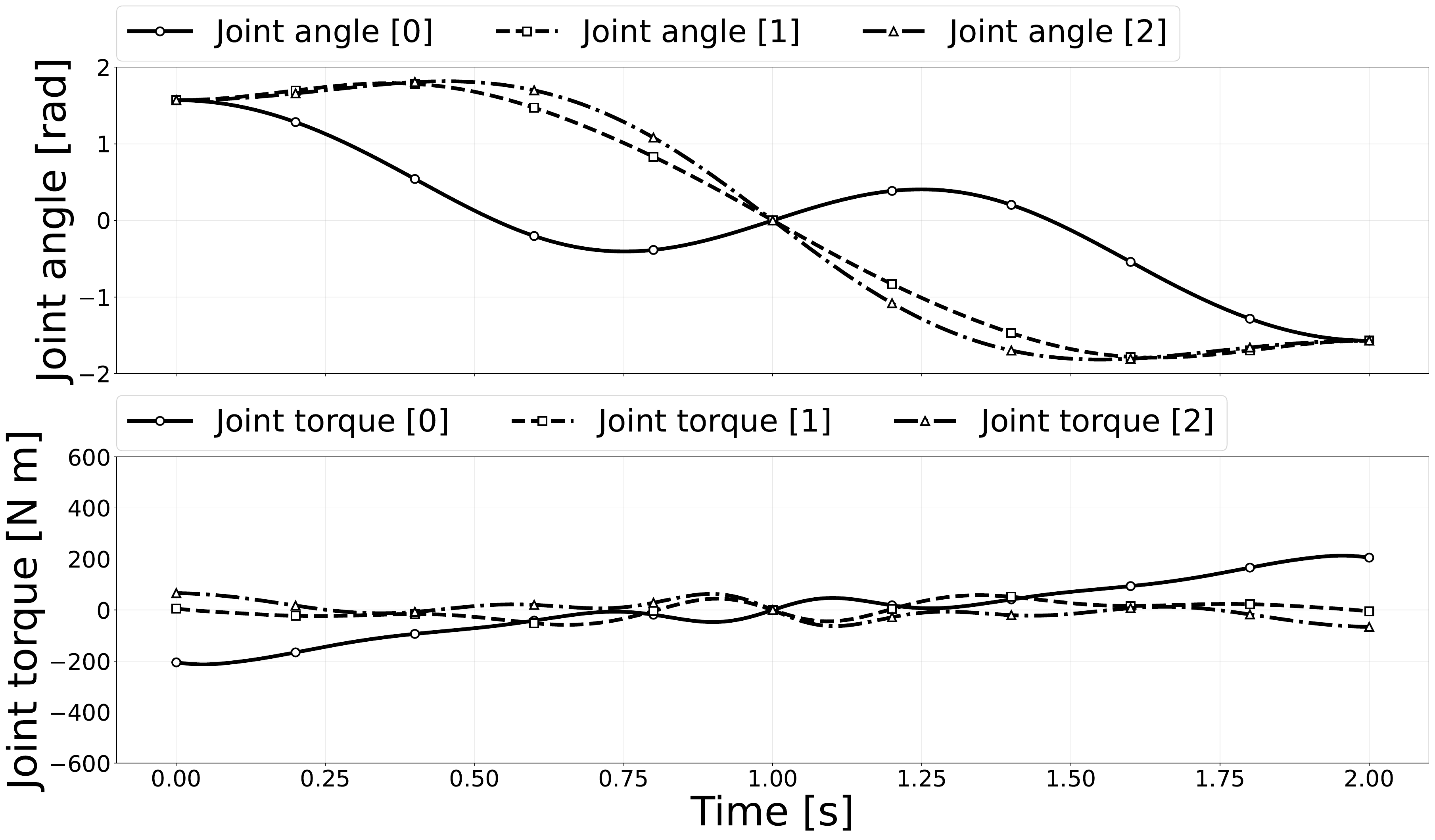}
  \caption{Condition (a) with additional torque minimization.}
\end{subfigure}

\vspace{0.5em}

\begin{subfigure}[b]{\columnwidth}
  \centering
  \includegraphics[width=\linewidth]{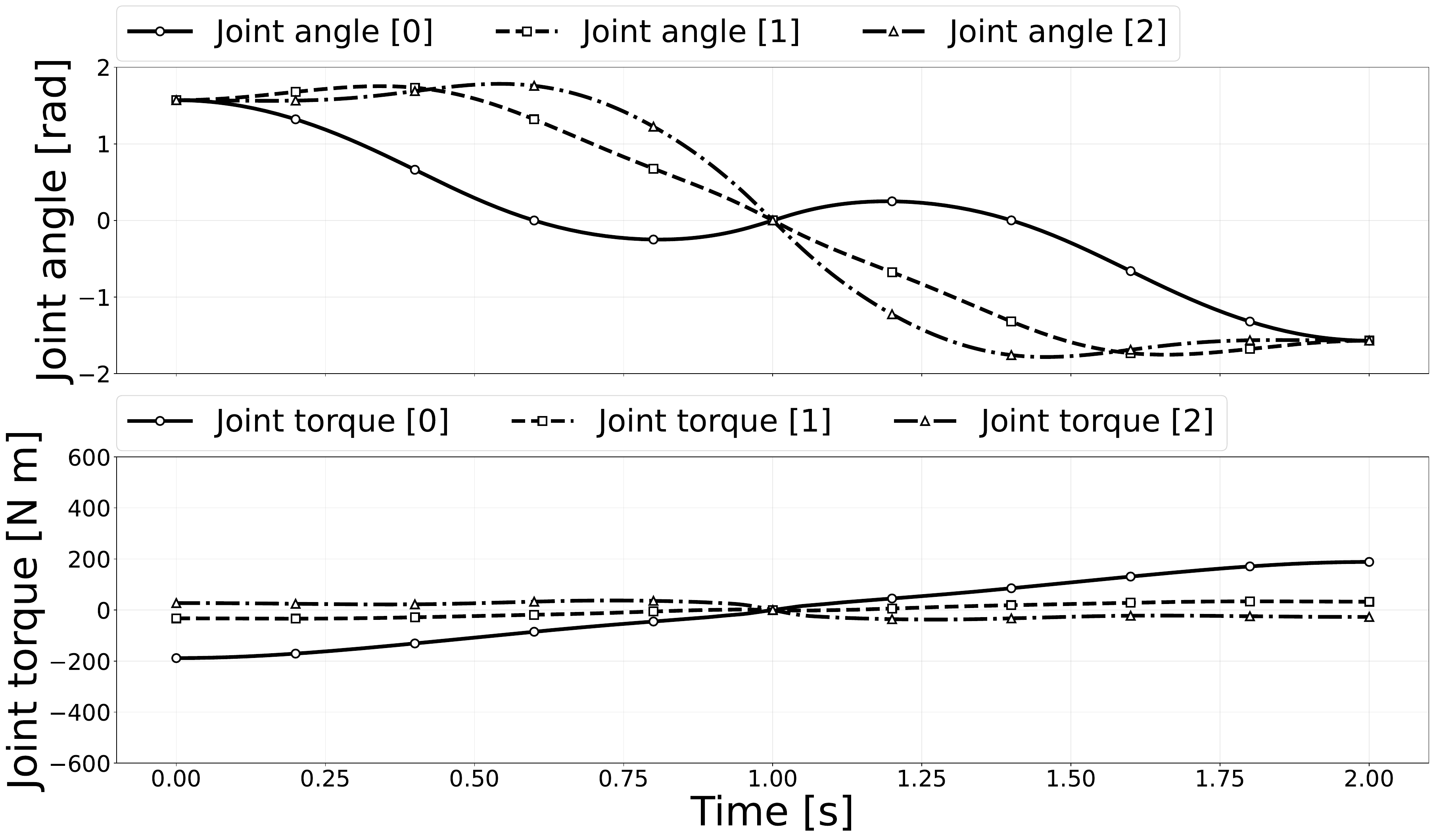}
  \caption{Condition (b) with additional minimization of torque rate.}
\end{subfigure}

\caption{
Motion optimization results for a 3-DOF planar robot under progressively augmented cost functions.
}
\label{fig:cost_vertical}
\end{figure}

\begin{figure}
    \centering
    \includegraphics[width=\linewidth]{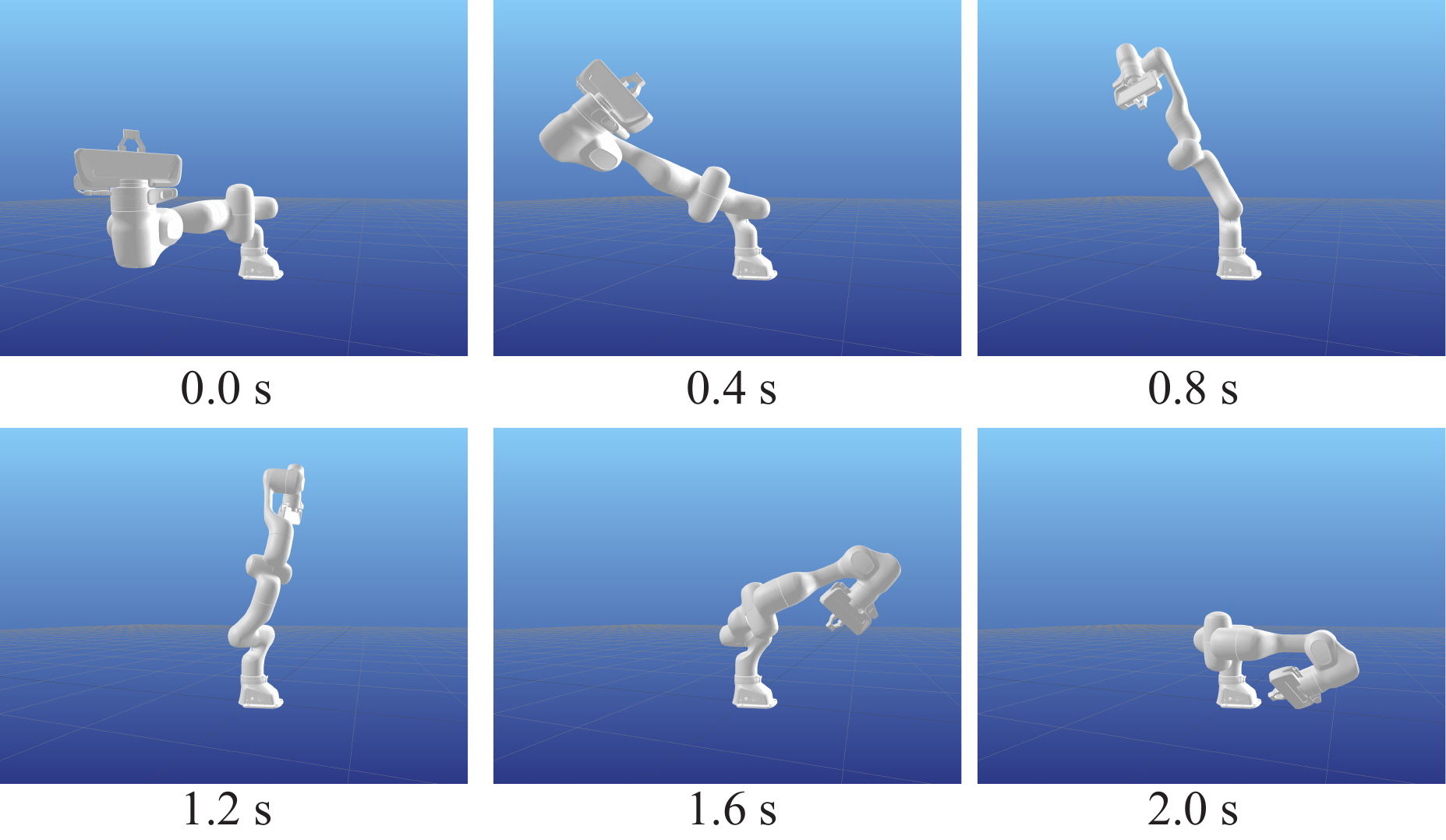}
    \caption{Snapshots of the simulation for the 7-DoF robot.}
    \label{fig:7link_robotsim}
\end{figure}

\begin{figure}
    \centering    
    \includegraphics[width=0.9\linewidth]{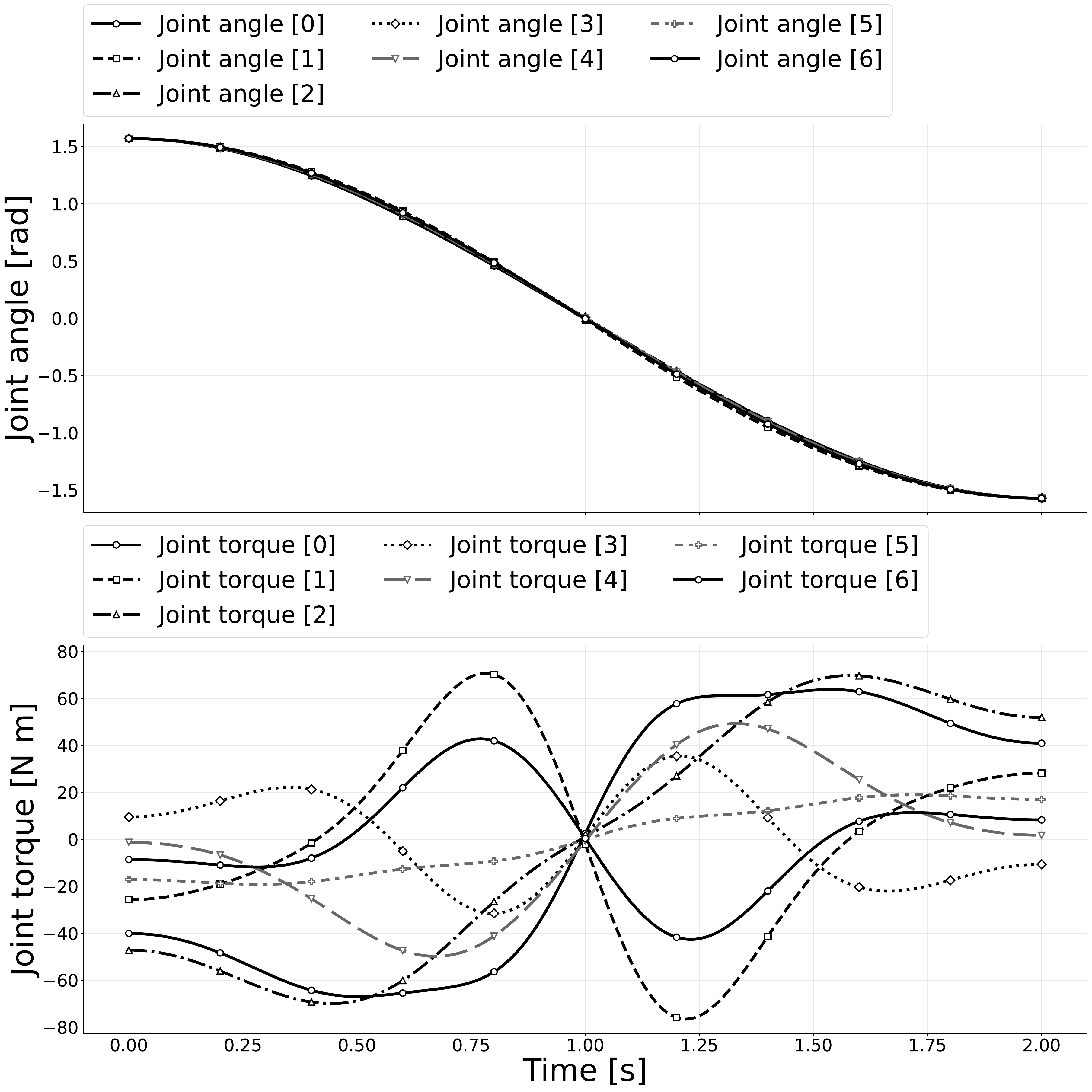}
    \caption{Motion optimization results for a 7-DOF robot.}
    \label{fig:7dof_robot_doc_result}
\end{figure}













\section{Conclusion}

In this paper, we achieved the following contributions:

\begin{enumerate}
\item In this paper, we presented a structured formulation for Jacobian computation in motion optimization involving higher-order time derivatives in multi-link systems.
The main theoretical components of the proposed framework are summarized as follows:

\begin{itemize}

\item Higher-order kinematic and dynamic quantities were systematically represented using the comprehensive motion transformation matrix (CMTM), enabling unified treatment of quantities and their time derivatives along the multi-link structure.

\item Tangent representations of CMTM were introduced to formulate the relationship between higher-order time derivatives and variations.

\item Structured variational relations were formulated for physical quantities ranging from generalized coordinates to generalized forces through propagations over both the multi-link structure and the dependency relations among physical quantities.

\item Jacobian computation was formulated as a structured propagation of variations through compositions of linear mappings associated with dependencies among physical quantities.

\item We constructed a structured optimization framework
for motion optimization involving higher-order time-derivative
physical quantities in multi-link systems, and demonstrated
its applicability to both direct and inverse optimization problems.
\end{itemize}

\item The proposed analytical Jacobian framework was evaluated through both numerical validation of analytical Jacobians and applications to direct and inverse motion optimization problems involving higher-order time-derivative quantities.
The main experimental findings are summarized as follows:

\begin{itemize}

\item The proposed analytical Jacobians achieved accuracy comparable to automatic differentiation across various robot configurations, while maintaining discrepancies with respect to numerical differentiation within the order of $10^{-9}$ to $10^{-7}$ depending on the derivative order.
The results further indicated that the dominant source of discrepancy originates from numerical differentiation rather than from the proposed analytical formulation itself.

\item The proposed framework was applied to direct and inverse motion optimization problems on 3-DoF and 7-DoF robot arms using kinematic and dynamic cost functions involving velocities, accelerations, jerks, torques, and torque rates.

\item Successful inverse optimal control was demonstrated using the proposed analytical Jacobians, recovering cost weights with estimation errors on the order of $10^{-8}$--$10^{-6}$.

\end{itemize}
\end{enumerate}

These results demonstrate that the proposed framework provides an efficient and reliable foundation for motion optimization involving higher-order dynamics. 
In particular, by explicitly exploiting the multi-link structure and the dependency among physical quantities, the proposed formulation enables systematic propagation of Jacobians across kinematic and dynamic variables, which is not directly achieved by conventional differentiation-based approaches.

The proposed framework also opens up several directions for future work. 
First, the structured Jacobian formulation can be integrated with learning-based methods, such as neural networks, enabling hybrid approaches that combine model-based optimization with data-driven representations. 
Second, the framework can be extended to more complex systems involving contact, closed-loop kinematics, or flexible bodies. 
In particular, extensions to flexible rod models and their associated differentiation and optimization problems, as studied in \cite{ishigaki2024comprehensive, ishigaki2025unified, mathew2025analytical}, are promising directions, where higher-order consistency and structured computation play a crucial role. 
Finally, the proposed formulation provides a basis for scalable optimization in high-dimensional systems, suggesting its applicability to whole-body control and large-scale robotic systems.

\section*{Funding}
This work was supported by JSPS KAKENHI Grant Numbers JP22H05002, JP25K21310.

\bibliographystyle{plain}
\bibliography{bib/ijrr_doc_ioc, bib/robotics_computation, bib/inverse_optimal_control}

\appendix
\section{Mathematical Foundations}

\subsection{Calculation of inverse variation \label{app:inv_var}}

\paragraph{Variation of the inverse matrix.}
For a Lie group matrix $\bm{X}$, the variation of its inverse is given by
\begin{align}
    \delta \bm{X}^{-1} 
    = 
    - \bm{X}^{-1} \delta \bm{X} \bm{X}^{-1},
\end{align}
which is obtained by taking the variation of $\bm{X}\bm{X}^{-1}=\eyem$ as
\begin{align}
    \delta \bm{X} \bm{X}^{-1} + \bm{X} \delta \bm{X}^{-1} = \zerom,
\end{align}
and solving for $\delta \bm{X}^{-1}$.

\paragraph{Variation of transformed vector.}
The variation of the transformation $\bm{X}^{-1}$ applied to a vector $\bm{v}$ is given by
\begin{align}
    \label{var_inv_trans_vec}
    (\delta \bm{X}^{-1} )\bm{v}
    =
    - \bm{X}^{-1} \delta \bm{X} \bm{X}^{-1} \bm{v}
    =
    - [(\bm{X}^{-1}\bm{v}) \hat{\times}] \delta \bm{x},
\end{align}
where $\delta \bm{X} = \bm{X}[\delta \bm{x} \times]$ is used.

\subsection{Numerical derivative of Lie group matrix}

Consider the mapping
\begin{align}
    \bm{y}(\bm{x}) = \bm{X}(\bm{x}) \bm{v},
\end{align}
where $\bm{X}(\bm{x}) \in G$ is parameterized by Lie algebra coordinates 
$\bm{x} \in \mathbb{R}^n$, and $\bm{v}$ is fixed.

At $\bm{x} = \bm{0}$, we have
\begin{align}
    \bm{y}(\bm{0}) = \bm{X}\bm{v}.
\end{align}

For a small perturbation $\delta \bm{x}$, the group element is approximated as
\begin{align}
    \bm{X}(\delta \bm{x})
    \approx
    \bm{X}(\bm{I} + [\delta \bm{x} \times]).
\end{align}
Thus, the perturbed output becomes
\begin{align}
    \bm{y}(\delta \bm{x})
    &=
    \bm{X}(\delta \bm{x}) \bm{v} \\
    &\approx
    \bm{X}(\bm{I} + [\delta \bm{x} \times]) \bm{v} 
    =
    \bm{X}\bm{v} + \bm{X}[\delta \bm{x} \times]\bm{v}.
\end{align}

Therefore, the variation is given by
\begin{align}
    \delta \bm{y}
    =
    \bm{y}(\delta \bm{x}) - \bm{y}(\bm{0})
    =
    \bm{X} [\delta \bm{x} \times] \bm{v}.
\end{align}

To compute the Jacobian numerically, we evaluate this variation along each basis direction.
Let $\bm{e}_i$ denote the $i$-th basis vector and set
\begin{align}
    \delta \bm{x} = \epsilon \bm{e}_i,
\end{align}
where $\epsilon$ is a small scalar.
Substituting this into the variation yields
\begin{align}
    \delta \bm{y}_i
    =
    \bm{X} [\epsilon \bm{e}_i \times] \bm{v}.
\end{align}

The $i$-th column of the Jacobian is then approximated by
\begin{align}
    \frac{\partial \bm{y}}{\partial x_i}
    \approx
    \frac{\delta \bm{y}_i}{\epsilon}
    =
    \bm{X} [\bm{e}_i \times] \bm{v}.
\end{align}

Finally, by stacking these column vectors, the Jacobian matrix is obtained as
\begin{align}
    \frac{\partial \bm{y}}{\partial \bm{x}}
    &=
    \begin{bmatrix}
        \dfrac{\partial \bm{y}}{\partial x_1} &
        \dfrac{\partial \bm{y}}{\partial x_2} &
        \cdots &
        \dfrac{\partial \bm{y}}{\partial x_n}
    \end{bmatrix}
    \notag
    \\
    &\approx
    \begin{bmatrix}
        \bm{X} [\bm{e}_1 \times] \bm{v} &
        \bm{X} [\bm{e}_2 \times] \bm{v} &
        \cdots &
        \bm{X} [\bm{e}_n \times] \bm{v}
    \end{bmatrix}.
\end{align}
\section{Properties of CMTM and Tangent Mapping}
\subsection{Inverse of CMTM}
The inverse of the CMTM can be computed as follows
\begin{align}
    \cmtm^{-1} &:= \cmtmat{\cmtpm^{-}},
    \\
    &\begin{cases}
        \cmtpm^{-}_0 = \spm^{-1}
        \\
        \cmtpm^{-}_{\ell+1}
        = 
        \frac{-1}{\ell+1} 
        \sum_{m=0}^\ell
        \adjsp{\spv^{(m|t!)}} \cmtpm^{-}_{\ell-m}
     \end{cases}.
\end{align}
As a property of the inverse matrix $\cmtm^{-1}$, the following relation holds:
\begin{align}
    \cmtm \cmtm^{-1} = \cmtm^{-1} \cmtm = \eyem.
\end{align}
On the other hand, it should be noted that the following relation holds when $k \neq 0$,
\begin{subequations}
\begin{align}
&\cmtpm_k^{-} \cmtpm_k \neq \eyem,\\
&\cmtpm_k \cmtpm_k^{-} \neq \eyem,\\
&\cmtpm_k \cmtpm_k^{-} \neq \cmtpm_k^{-} \cmtpm_k .
\end{align}
\end{subequations}
\subsection{CMTM Tangent Matrix \label{apdx_cmtm_tan_map}}
The explicit form of $\cmtanmap$ is given as follows
\begin{align}
    &\cmvec{\cmtanmap}{(\dindx)} =
    \begin{bmatrix}
        \cmtanmap_{00} & \zerom_6 & \cdots \\
        \cmtanmap_{10} & \cmtanmap_{11} & \zerom_6
        & \cdots \\
        \vdots & \vdots & \ddots \\
        \cmtanmap_{k0} & \cmtanmap_{k1} & 
        \cdots & \cmtanmap_{kk}
    \end{bmatrix}
    \\
    &=
    \begin{bmatrix}
        \eyem_6 & \zerom_6 & \cdots \\
        \adjsp{\spv^{(0)}} & \eyem_6 & \zerom_6 & \cdots \\
        \adjsp{\spv^{(1)}} & \adjsp{\spv^{(0)}} & 2\eyem_6 & \zerom_6 & \cdots \\
        \vdots & \vdots & \ddots & \ddots & \ddots \\
        \adjsp{\spv^{(\dindx)}} & \adjsp{\spv^{(\dindx-1)}} & \cdots & \adjsp{\spv^{(0)}} & \dindx\eyem_6 \\
    \end{bmatrix},
    \notag
    \\
    \quad&\begin{cases}
        \cmtanmap_{00} = \eyem_6 \\
        \cmtanmap_{\ell\ell} = \ell\eyem_6 \\
        \cmtanmap_{\ell m} = \adjsp{\spv^{(\ell-m)}}
    \end{cases}.
\end{align}
Since $\cmtanmap$ is a lower block triangular matrix and admits an inverse,  
$\cmtanmap^{-1}$ can be computed recursively as follows
\begin{align}
    \cmtanmap^{-1} &=
    \begin{bmatrix}
        \cmtanmap_{00}^{-} & \zerom_6 & \cdots \\
        \cmtanmap_{10}^{-} & \cmtanmap_{11}^{-} & \zerom_6
        & \cdots \\
        \vdots & \vdots & \ddots \\
        \cmtanmap_{k0}^{-} & \cmtanmap_{k1}^{-} & 
        \cdots & \cmtanmap_{kk}^{-}
    \end{bmatrix}.
\end{align}
This structure reflects the causal dependency of higher-order tangent quantities on lower-order ones, allowing the inverse mapping to be constructed in a forward substitution manner
\begin{align}
    \cmtanmap_{\ell\ell}^{-} &= \cmtanmap_{\ell\ell}^{-1} ,\\
    \cmtanmap_{\ell m}^{-}
    &= 
    -\cmtanmap_{\ell\ell}^{-1} 
    \sum_{n = m}^{\ell-1}
    \cmtanmap_{\ell n} \cmtanmap_{n m}^{-}
    \notag
    \\
    &=         
    -\frac{1}{l}
    \sum_{n = m}^{\ell-1}
    \adjsp{\spv^{(\ell-n-1)}}
    \cmtanmap_{n m}^{-}.
\end{align}
Each block $\cmtanmap_{\ell m}^{-}$ represents the contribution of the $m$-th order variation to the $\ell$-th order tangent component.
Using this inverse mapping, the stacked tangent vector up to order $(\dindx+1)$ can be recovered from the variation of the base quantities
\begin{align}
    \label{cm_tan_map_inv}
    \cmvec{\cmtanv}{(\dindx+1)}
    =
    \cmvec{\cmtanmap}{(\dindx)}^{-1}
    \begin{bmatrix}
        \delta \bm{a} \\
        \delta \cmvec{\cmv}{(\dindx)}
    \end{bmatrix}.
\end{align}

\section{Whole-body comprehensive motion computation \label{app:wbcmtm}}
We formulate a unified linear structure that describes the relationships among whole-body quantities, including spatial velocities, momenta, forces, and joint torques.
These quantities are represented as stacked vectors over all links and joints, and their relationships are expressed through structured block matrices induced by the kinematic tree.
This formulation reveals that all physical quantities can be computed through a sequence of linear mappings sharing a common sparsity pattern.

Here, $\dof$ denotes the number of links.  
The vectors obtained by collecting a physical quantity $\bm{x}$ associated with each link and each joint are expressed as follows
\begin{align}
    \label{def_total_express}
    \bm{x}_\link
    &=
    \begin{bmatrix}
        \bm{x}_{\link_0}^\trans & \cdots & \bm{x}_{\link_\dof}^\trans
    \end{bmatrix}^\trans,
    \\
    \bm{x}_\joint
    &=
    \begin{bmatrix}
        \bm{x}_{\joint_0}^\trans & 
        \cdots & \bm{x}_{\joint_n}^\trans
    \end{bmatrix}^\trans.
\end{align}
The matrix ${}^\joint\cmtm_\link$ encodes the local kinematic dependencies between parent and child links into a global linear operator.
Each row corresponds to a joint, and is constructed only from its own link and its parent link, resulting in a sparse block structure.

The relationship among the spatial velocities of all links and joints in the robot can be expressed as follows
\begin{align}
    \label{wcm_rel_jl}
    \cmv_\joint = {}^\joint\cmtm_\link \cmv_\link.
\end{align}
Here, an element ${{}^\joint\cmtm_\link}_{\{i,j\}}$ of ${}^\joint\cmtm_\link$ is expressed as follows
\begin{align}
    {{}^\joint\cmtm_\link}_{\{i,j\}}
    :=
    \begin{cases}
        \eyem & (i = j) \\
        -{}^i\cmtm_j & (j = \parent{i}) \\
        \zerom & \text{otherwise}.
    \end{cases}
\end{align}
$p(i)$ denotes the index of the parent link of link $i$.
In the case of a serial-link system, \eqref{wcm_rel_jl} take the form as
\begin{align}
    \label{serical_wcm_jl}
    \underbrace{
    \begin{bmatrix}
            \cmv_{\joint_0} \\
            \cmv_{\joint_1} \\
            \cmv_{\joint_2} \\
		\vdots \\
		\cmv_{\joint_n}
	\end{bmatrix}
    }_{\cmv_\joint}
	&= 
    \underbrace{
	\begin{bmatrix}
            \eyem & \zerom & \cdots & & \\
		-{}^{1}\cmtm_{0} & \eyem & \zerom & \cdots & \\
		\zerom & -{}^{2}\cmtm_{1} & \eyem & \zerom & \cdots \\
		\vdots & \ddots & \ddots & \ddots & \ddots & \\
		\zerom & \cdots & \zerom & -{}^{n}\cmtm_{n-1} & \eyem
	\end{bmatrix}
    }_{{}^{\joint}\cmtm_{\link}}
    \underbrace{
	\begin{bmatrix}
		\cmv_{\link_0}  \\
		\cmv_{\link_1} \\
            \cmv_{\link_2} \\
		\vdots\\
		\cmv_{\link_n}
	\end{bmatrix}
    }_{\cmv_\link}.
\end{align}
An element of inverse mapping of ${}^{\joint}\cmtm_{\link}$, denoted by ${{}^{\link}\cmtm_{\joint}}_{\{i,j\}}$, is expressed as follows,
\begin{align}
    {{}^\link\cmtm_\joint}_{\{i,j\}}
    :=
    \begin{cases}
        \eyem & (i = j) \\
        -{}^i\cmtm_j & (j = \parents{i}) \\
        \zerom & \text{otherwise}.
    \end{cases}
\end{align}
Same as \eqref{serical_wcm_jl}, 
in the case of a serial-link system without branching,
inverse relationship of \eqref{wcm_rel_jl} can be explicitly expressed as follows
\begin{align}
    \label{serical_wcm_lj}
    \underbrace{
	\begin{bmatrix}
		\cmv_{\link_0} \\
		\cmv_{\link_1} \\
            \cmv_{\link_2} \\
		\vdots\\
		\cmv_{\link_n}
	\end{bmatrix}
    }_{\cmv_\link}
	= 
    \underbrace{
	\begin{bmatrix}
            \eyem & \zerom & \cdots & & \\
		{}^{1}\cmtm_{0} & \eyem & \zerom & \cdots & \\
		{}^{2}\cmtm_{0} & {}^{2}\cmtm_{1} & \eyem & \zerom & \cdots & \\
            \vdots & \vdots &  & \ddots & \ddots\\
            {}^{n}\cmtm_{0} & {}^{n}\cmtm_{1} & \cdots & {}^{n}\cmtm_{1} & \eyem \\
	\end{bmatrix}
    }_{{}^{\link}\cmtm_{\joint}}
    \underbrace{
	\begin{bmatrix}
            \cmv_{\joint_0} \\
            \cmv_{\joint_1} \\
            \cmv_{\joint_2} \\
		\vdots \\
		\cmv_{\joint_n}
	\end{bmatrix}
    }_{\cmv_\joint}.
\end{align}
Similarly to \eqref{def_cmtm_moment}, the relationship between $\cmmom_{\link}$ and $\cmv_{\link}$ can be expressed as follows
\begin{align}
    \cmmom_{\link}
    &=
    \cminertia_{\link}
    \cmv_{\link}, \\
    \cminertia_{\link}
    &:=
    \text{diag}(
    \begin{bmatrix}
        \cminertia_0 &
        \cdots &
        \cminertia_n
    \end{bmatrix}
    ).
\end{align}

From \eqref{cm_moment_dynamics}, the moment can also be expressed as follows
\begin{align}
    \label{total_momentum_link_joint_relation}
    \cmmom_\link 
    &= 
    {}^\link\cmtm_\joint^{\ast} 
    \cmmom_\joint,
    \\
    {{}^\link\cmtm_\joint^{\ast}}_{\{i,j\}}
    &:=
    \begin{cases}
        \eyem & (i = j) \\
        -{}^i\cmtm_j^{\ast} & (j = \child{i}) \\
        \zerom & \text{otherwise}.
    \end{cases}
\end{align}
$\child{i}$ denotes the index of the child link of link $i$.
As a concrete example for a serial-link system, this relationship can be expressed as
\begin{align}
    \underbrace{
	\begin{bmatrix}
            \cmmom_{\link_0} \\
            \cmmom_{\link_1} \\
		\vdots \\
            \cmmom_{\link_{n-1}} \\
		\cmmom_{\link_n}
	\end{bmatrix}
    }_{\cmmom_\link}
	&= 
    \underbrace{
	\begin{bmatrix}
            \eyem & -{}^{0}\cmtm_{1}^{\ast} & \zerom & \cdots  \\
		  \zerom & \eyem & -{}^{1}\cmtm_{2}^{\ast}  
            & \zerom & \cdots \\
		\vdots & \ddots & \ddots & \ddots & \ddots  \\
            \zerom & \cdots & \zerom & \eyem & -{}^{n-1}\cmtm_{n}^{\ast} \\
		\zerom & & \cdots & \zerom & \eyem
	\end{bmatrix}
    }_{{}^{\link}\cmtm_\joint^{\ast}}
    \underbrace{
	\begin{bmatrix}
		\cmmom_{\joint_0}  \\
		\cmmom_{\joint_1} \\
		\vdots\\
            \cmmom_{\joint_{n-1}} \\
		\cmmom_{\joint_n}
	\end{bmatrix}
    }_{\cmmom_\joint}.
\end{align}
Its inverse mapping is given as follows
\begin{align}
    {{}^\joint\cmtm_\link^{\ast}}_{\{i,j\}}
    &:=
    \begin{cases}
        \eyem & (i = j) \\
        {}^i\cmtm_j^{\ast} & (j = \children{i}) \\
        \zerom & \text{otherwise}.
    \end{cases}
\end{align}
In the case of serial-link system, inverse relationship of \eqref{total_momentum_link_joint_relation} as follows
\begin{align}
    \underbrace{
	\begin{bmatrix}
		\cmmom_{\joint_0} \\
		\cmmom_{\joint_1} \\
		\vdots\\
            \cmmom_{\joint_{n-1}} \\
		\cmmom_{\joint_n}
	\end{bmatrix}
    }_{\cmmom_\joint}
	&= 
    \underbrace{
	\begin{bmatrix}
            \eyem & 
            {}^{0}\cmtm_{1}^{\ast} & \cdots & {}^{0}\cmtm_{n-1}^{\ast} & {}^{0}\cmtm_{n}^{\ast} \\
		      \zerom & \eyem & {}^{1}\cmtm_{2}^{\ast} & 
            \cdots & {}^{1}\cmtm_{n}^{\ast} \\
            \vdots & \ddots & \ddots & \ddots & \vdots\\
            \zerom & \cdots & \zerom & \eyem & {}^{n-1}\cmtm_{n}^{\ast}\\
            \zerom &  & \cdots & \zerom & \eyem \\
	\end{bmatrix}
    }_{{}^\joint\cmtm_\link^{\ast}}
    \underbrace{
	\begin{bmatrix}
            \cmmom_{\link_0} \\
            \cmmom_{\link_1} \\
            \vdots \\
            \cmmom_{\link_{n-1}} \\
		\cmmom_{\link_n}
	\end{bmatrix}
    }_{\cmmom_\link}.  
\end{align}

With respect to the moment expressed in the world coordinate frame, it can be formulated as follows based on \eqref{wcm_moment_dynamics}
\begin{align}
\label{total_world_momentum_link_joint_relation}
    \wcmmom_{\link}
    &=
    {}^\link\wcmtm_\joint^{\ast} 
    \wcmmom_{\joint},
    \\
    \label{total_world_momentum_link_joint_map}
    {{}^\link\wcmtm_\joint^{\ast} }_{\{i,j\}}
    &:=
    \begin{cases}
        \eyem & (i = j) \\
        -\eyem & (j = \children{i}) \\
        \zerom & \text{otherwise}.
    \end{cases}
\end{align}
Since \eqref{wcm_moment_dynamics} is expressed as a simple summation, ${{}^\link\wcmtm_\joint^{\ast} }_{\{i,j\}}$ is composed of identity matrices and zero matrices.
In the world coordinate representation, the transformation matrices are further simplified into combinations of identity and zero matrices.
This eliminates dependency on local coordinate transformations and results in a purely additive structure, which is particularly advantageous for differentiation and optimization.
As a concrete example for a serial-link system, this relationship is as
\begin{align}
\underbrace{
	\begin{bmatrix}
        \wcmmom_{\link_0} \\
        \wcmmom_{\link_1} \\
		\vdots \\
        \wcmmom_{\link_{n-1}} \\
		\wcmmom_{\link_n}
	\end{bmatrix}
    }_{\wcmmom_\link}
	&= 
    \underbrace{
	\begin{bmatrix}
        \eyem & -\eyem & \zerom & \cdots & \\
		  \zerom & \eyem & -\eyem & \zerom & \cdots \\
		\vdots & \ddots & \ddots & \ddots & \ddots & \\
        \zerom & \cdots & \zerom & \eyem & -\eyem \\
		\zerom & & \cdots & \zerom & \eyem
	\end{bmatrix}
    }_{{}^\link\wcmtm_\joint}
    \underbrace{
	\begin{bmatrix}
		\wcmmom_{\joint_0}  \\
		\wcmmom_{\joint_1} \\
		\vdots\\
        \wcmmom_{\joint_{n-1}} \\
		\wcmmom_{\joint_n}
	\end{bmatrix}
    }_{\wcmmom_\joint},
\end{align} 
and the inverse relationship is as
\begin{align}
\underbrace{
	\begin{bmatrix}
		\wcmmom_{\joint_0}  \\
		\wcmmom_{\joint_1} \\
		\vdots\\
        \wcmmom_{\joint_{n-1}} \\
		\wcmmom_{\joint_n}
	\end{bmatrix}
    }_{\wcmmom_\joint}
	&= 
    \underbrace{
	\begin{bmatrix}
        \eyem & \eyem & \cdots & \cdots & \eyem \\
		  \zerom & \eyem & \eyem & \cdots & \eyem\\
		\vdots & \ddots & \ddots & \ddots & \vdots & \\
        \zerom & \cdots & \zerom & \eyem & \eyem \\
		\zerom & & \cdots & \zerom & \eyem
	\end{bmatrix}
    }_{{}^\joint\wcmtm_\link}
    \underbrace{
    \begin{bmatrix}
        \wcmmom_{\link_0} \\
        \wcmmom_{\link_1} \\
		\vdots \\
        \wcmmom_{\link_{n-1}} \\
		\wcmmom_{\link_n}
	\end{bmatrix}
    }_{\wcmmom_\link}.
\end{align}
The transformation between the local coordinate frame and the world coordinate frame is given by the following equations
\begin{align}
    \label{total_world_moment_trans}
    \wcmmom_{\link} 
    &=
    {}^{\world}\cmtm_{\link}^{\ast}
    \cmmom_{\link},
   \\
   {}^{\text{w}}\cmtm_{\link} &:=
   \text{diag}
   (
   \begin{bmatrix}
       \cmtm_{\link_0} &
       \cdots &
       \cmtm_{\lnum} 
   \end{bmatrix}
   ) ,
    \\
    \label{total_world_joint_moment_trans}
    \wcmmom_{\joint} 
    &=
    {}^{\text{w}}\cmtm_{\joint}^{\ast}
    \cmmom_{\joint},
   \\
   {}^{\text{w}}\cmtm_{\joint} &:=
   \text{diag}
   (
   \begin{bmatrix}
       \cmtm_{\joint_0} &
       \cdots &
       \cmtm_{\jnum}
   \end{bmatrix}
   ).
\end{align}
By using \eqref{total_world_momentum_link_joint_relation} and \eqref{total_world_moment_trans},  
the following equation, which is equivalent to \eqref{total_momentum_link_joint_relation}, is obtained
\begin{align}
   \label{total_momentum_link_joint_relation2}
   \cmmom_{\link} 
   &=
   {{}^{\world}\cmtm_{\link}^{\ast}}^{-1}
   ({}^\link\wcmtm_\joint^{\ast})
   {}^{\world}\cmtm_{\joint}^{\ast}
   \cmmom_{\joint}.
\end{align}
That is, by comparing \eqref{total_momentum_link_joint_relation} and \eqref{total_momentum_link_joint_relation2}, the following relationship holds
\begin{align}
   {}^\link\cmtm_\joint^{\ast}
   &=
   {{}^{\world}\cmtm_{\link}^{\ast}}^{-1}
   ({}^\link\wcmtm_\joint^{\ast})
   {}^{\world}\cmtm_{\joint}^{\ast}.
\end{align}
\eqref{total_momentum_link_joint_relation2} expresses the transformation matrix in terms of the composite motion transformation matrices of each link observed from the world coordinate frame, together with identity and zero matrices. This representation facilitates the subsequent differentiation required for optimization.

Similarly, the transformation from momentum to force described in \eqref{cm_force_momentum}, \eqref{link_mom_to_force}, and \eqref{joint_mom_to_force} can be computed as follows
\begin{align}
    \label{total_force_moment}
    \cmvec{{\cmforce_\link}}{(\dindx)}
    &=
    \cmvec{{\cnatm_\link}}{(\dindx)}
    \cmvec{{\cmmom_\link}}{(\dindx+1)},
    \\
    \cmvec{{\cmforce_\joint}}{(\dindx)}
    &=
    \cmvec{{\cnatm_\joint}}{(\dindx)}
    \cmvec{{\cmmom_\joint}}{(\dindx+1)},
    \\
    \cmvec{{\cnatm_\link}}{(\dindx)} 
    &:=
    \text{diag}(
    \begin{bmatrix}
        \bm{\mathfrak{U}}(\cadjcm{\cmvec{{\cmv_{\link_0}}}{(\dindx)}{}})
        & \cdots & 
        \bm{\mathfrak{U}}(\cadjcm{\cmvec{{\cmv_{\link_n}}}{(\dindx)}{}})
    \end{bmatrix}
    ),
    \\
    \cmvec{{\cnatm_\joint}}{(\dindx)} 
    &:=
    \text{diag}(
    \begin{bmatrix}
        \bm{\mathfrak{U}}(\cadjcm{\cmvec{{\cmv_{\link_0}}}{(\dindx)}{}})
        & \cdots & 
        \bm{\mathfrak{U}}(\cadjcm{\cmvec{{\cmv_{\link_n}}}{(\dindx)}{}})
    \end{bmatrix}.
\end{align}
The following equation also holds for the joint torques.
\begin{align}
    \label{total_torque_force}
    \cmtorque_{\joint} 
    &=
    \cmslctm_{\joint}^{\text{T}}
    \cmforce_{\joint},
    \\
    \cmslctm_{\joint}
    &:=
    \text{diag}(
    \begin{bmatrix}
        \cmslctm_{\joint_0} &
        \cdots &
        \cmslctm_{\jnum}
    \end{bmatrix}
    ).
\end{align}

All subsequent quantities, including momentum, force, and joint torque, can be derived using the same structured linear mappings.
The differences lie only in the choice of transformation operators, while the underlying computational structure remains identical.
\section{B-spline Trajectory Representation \label{app:bspline}}
In motion optimization involving higher-order time-derivative physical quantities, joint trajectories are often represented using B-splines in order to ensure smoothness of the trajectory and its higher-order derivatives while maintaining a compact parametrization \cite{ayusawa2019predictive,suleiman2008human}. 
The optimization variables are defined as the control points of the spline representation.

\subsection{B-spline Basis Functions}

Let the joint trajectory be represented as
\begin{align}
\bm{q}(t)
=
\sum_{i=0}^{n_c-1}
b_{i,p}(t)\bm{m}_i,
\end{align}
where $b_{i,p}(t)\in\mathbb{R}$ denotes the $p$-th degree B-spline basis function, $\bm{m}_i\in\mathbb{R}^{n_q}$ is the $i$-th control point, $n_c$ is the number of control points, and $n_q$ is the number of generalized coordinates.

The B-spline basis functions are recursively defined using the Cox--de Boor formula:
\begin{align}
b_{i,0}(t)
&=
\begin{cases}
1 & (u_i \le t < u_{i+1}), \\
0 & \text{otherwise},
\end{cases}
\\
b_{i,p}(t)
&=
\frac{t-u_i}{u_{i+p}-u_i}
b_{i,p-1}(t)
+
\frac{u_{i+p+1}-t}{u_{i+p+1}-u_{i+1}}
b_{i+1,p-1}(t),
\end{align}
where $\{u_i\}$ denotes the knot vector.

\subsection{Higher-Order Derivatives}

The derivatives of the B-spline basis functions can also be computed recursively \cite{piegl1997nurbs}:
\begin{align}
\frac{\partial}{\partial t}b_{i,p}(t)
=
\frac{p}{u_{i+p}-u_i}
b_{i,p-1}(t)
-
\frac{p}{u_{i+p+1}-u_{i+1}}
b_{i+1,p-1}(t).
\end{align}

Higher-order derivatives are obtained recursively by repeated application of the derivative relation:
\begin{align}
\parvec{b_{i,p}(t)}{k}{t}
&=
\frac{p}{u_{i+p}-u_i}
\parvec{b_{i,p-1}(t)}{k-1}{t}
\nonumber\\
&\quad-
\frac{p}{u_{i+p+1}-u_{i+1}}
\parvec{b_{i+1,p-1}(t)}{k-1}{t}.
\end{align}

Assuming simple knots, a $p$-th degree B-spline guarantees $C^{p-1}$ continuity.
Therefore, trajectories whose $k$-th derivatives remain continuous require a spline degree satisfying $p \ge k+1$.
This property enables smooth representations of higher-order time-derivative quantities.

\subsection{Linear Parameterization}

By stacking all control points into a parameter vector
\begin{align}
\bm{\theta}
=
\begin{bmatrix}
\bm{m}_0^\top &
\bm{m}_1^\top &
\cdots &
\bm{m}_{n_c-1}^\top
\end{bmatrix}^\top,
\end{align}
the joint trajectory and its higher-order derivatives can be written as linear functions of the spline parameters:
\begin{align}
\parvec{\bm{q}(t)}{k}{t}
=
\parvec{\bm{D}_p(t)}{k}{t}
\bm{\theta},
\end{align}
where $\parvec{\bm{D}_p(t)}{k}{t}$ denotes the basis matrix associated with the $k$-th derivatives of the B-spline basis functions:
\begin{align}
\label{eq:D_def}
&\parvec{\bm{D}_p}{k}{t}
=
\\
&
\begin{bmatrix}
\parvec{b_{0,p}}{k}{t}\eyem &
\parvec{b_{1,p}}{k}{t}\eyem &
\cdots &
\parvec{b_{n_c-1,p}}{k}{t}\eyem
\end{bmatrix}.
\notag
\end{align}

This linear representation enables efficient evaluation of higher-order time-derivative quantities in motion optimization while preserving continuity properties of the trajectory.

\end{document}